\documentclass[letterpaper]{article} % DO NOT CHANGE THIS
\usepackage{aaai2026}  % DO NOT CHANGE THIS
\usepackage{times}  % DO NOT CHANGE THIS
\usepackage{helvet}  % DO NOT CHANGE THIS
\usepackage{courier}  % DO NOT CHANGE THIS
\usepackage[hyphens]{url}  % DO NOT CHANGE THIS
\usepackage{graphicx} % DO NOT CHANGE THIS
\usepackage{natbib}  % DO NOT CHANGE THIS AND DO NOT ADD ANY OPTIONS TO IT
\usepackage{caption} % DO NOT CHANGE THIS AND DO NOT ADD ANY OPTIONS TO IT
\usepackage{algorithm}
\usepackage{algpseudocode}
\algrenewcommand\algorithmicrequire{\textbf{Input:}}
\algrenewcommand\algorithmicensure{\textbf{Output:}}

\usepackage{symbolDef}
\usepackage{multirow}
\usepackage{subcaption}
\usepackage{amsmath}
\usepackage{enumitem}
\usepackage{amsthm}
\usepackage{xcolor}
\usepackage{booktabs}
\newtheorem{theorem}{Theorem}
\newtheorem{assumption}{Assumption}
\newtheorem{proposition}{Proposition}
\newtheorem{lemma}{Lemma}
\newtheorem{definition}{Definition}

\newtheorem{property}{Property}

\usepackage{newfloat}
\usepackage{listings}
\DeclareCaptionStyle{ruled}{labelfont=normalfont,labelsep=colon,strut=off} % DO NOT CHANGE THIS
\floatstyle{ruled}
\newfloat{listing}{tb}{lst}{}
\floatname{listing}{Listing}
\title{Covariance Scattering Transforms}
\author{
    Andrea Cavallo\textsuperscript{\rm 1},
    Ayushman Raghuvanshi\textsuperscript{\rm 2},
    Sundeep Prabhakar Chepuri\textsuperscript{\rm 2},
    Elvin Isufi\textsuperscript{\rm 1}
}
\affiliations{
    \textsuperscript{\rm 1}Delft University of Technology, Delft, Netherlands\\
    \textsuperscript{\rm 2}Indian Institute of Science, Bangalore, India\\
   \{a.cavallo, e.isufi-1\}@tudelft.nl, \{ayushmanr, spchepuri\}@iisc.ac.in
}

\usepackage{bibentry}
\begin{document}

\maketitle

\begin{abstract}
Machine learning and data processing techniques relying on covariance information are widespread as they identify meaningful patterns in unsupervised and unlabeled settings.
As a prominent example, Principal Component Analysis (PCA) projects data points onto the eigenvectors of their covariance matrix, capturing the directions of maximum variance. This mapping, however, falls short in two directions: it fails to capture information in low-variance directions, relevant when, e.g., the data contains high-variance noise; and it provides unstable results in low-sample regimes, especially when covariance eigenvalues are close. 
CoVariance Neural Networks (VNNs), i.e., graph neural networks using the covariance matrix as a graph, show improved stability to estimation errors and learn more expressive functions in the covariance spectrum than PCA, but require training and operate in a labeled setup. 
To get the benefits of both worlds, we propose Covariance Scattering Transforms (CSTs), deep untrained networks that sequentially apply filters localized in the covariance spectrum to the input data and produce expressive hierarchical representations via nonlinearities. We define the filters as covariance wavelets that capture specific and detailed covariance spectral patterns. We improve CSTs' computational and memory efficiency via a pruning mechanism, and we prove that their error due to finite-sample covariance estimations is less sensitive to close covariance eigenvalues compared to PCA, improving their stability. 
Our experiments on age prediction from cortical thickness measurements on 4 datasets collecting patients with neurodegenerative diseases show that CSTs produce stable representations in low-data settings, as VNNs but without any training, and lead to comparable or better predictions w.r.t. more complex learning models. 
\end{abstract}

% Uncomment the following to link to your code, datasets, an extended version or similar.
% You must keep this block between (not within) the abstract and the main body of the paper.
\begin{links}
    \link{Code}{https://github.com/andrea-cavallo-98/CST}
    % \link{Datasets}{https://aaai.org/example/datasets}
    \link{Extended version}{https://arxiv.org/abs/2511.08878}
\end{links}

\section{Introduction}

Covariance information captures relevant data characteristics and is widely used to gain insights about data interdependencies, find latent relations and increase data processing performance in unsupervised settings. 
For example, correlations in brain activity recordings and cortical thickness measures are of high importance to identify interactions among brain regions and co-activation patterns, which lead to deeper understanding of neural dynamics and better neurodegenerative disease prediction~\cite{yin2023anatomically,bessadok2022graph,sihag2024explainable,bashyam2020mri}. 
% financial data forecasting relies on stocks' correlations to model their behavior~\cite{palomar2024portfolio,cardoso2020algorithms}, and sensor analysis benefits from inferring dependencies across measurements~\cite{liao2022har,vuran2004spatio}.
Often, the covariance information is accessed via the spectrum of the covariance matrix, which characterizes the concentration of variance along different directions. Principal Component Analysis (PCA), for example, projects the data on the covariance eigenvectors to maximize variance and potentially reduce dimensionality by selecting the directions corresponding to the largest eigenvalues~\cite{Jolliffe2002pca} -- this can be seen as an ideal high-pass filtering in the covariance spectrum, cf. Figure~\ref{fig:motivation_example}. 
However, data might exhibit complex patterns in the covariance spectrum that lead to relevant information localized in low-variance directions, which PCA filtration loses. 
Furthermore, PCA is heavily affected by estimation errors in low-sample regimes, causing its output to significantly diverge when the covariance estimation is not reliable~\cite{Jolliffe2002pca,Jolliffe2016PrincipalCA}. 
To mitigate these problems, the work in~\cite{sihag2022covariance} considers each feature of a data sample as a node of a graph and the covariances among two features as weighted edges, and introduces coVariance Neural Networks (VNNs), graph neural networks operating on this graph. VNNs learn polynomial filtering functions in the covariance eigenvalues, which gives them the flexibility to identify complex variance patterns, and are stable to finite-sample estimation errors. Such properties made VNNs effective in a variety of settings~\cite{sihag2024explainable,sihag2023transferablility,cavallo2024stvnn,cavallo2025fair}. 
However, VNNs rely on labeled data for training, and their stability and expressivity depend on their learning dynamics, which are difficult to control.  

To merge the benefits of both approaches -- unsupervised and untrained nature of PCA as well as the stability and expressivity of VNNs -- we propose Covariance Scattering Transforms (CSTs), deep architectures that process covariance information in a fully untrained manner with stability guarantees under finite-sample covariance estimation errors.
The building block of CSTs are covariance wavelets, localized functions that capture specific patterns in the covariance eigenvalues, as shown in Figure~\ref{fig:motivation_example}, for which we propose three distinct implementations. 
CSTs sequentially apply banks of wavelet filters followed by nonlinearities and, optionally, low-pass aggregations to decrease representation size.
To reduce the computation and coefficients of CSTs and improve stability, we apply a pruning mechanism that removes wavelet coefficients carrying low energy, deemed irrelevant for efficient representations.
Our contributions are summarized as follows.
\begin{itemize}
    \item We coin the concept of CSTs, untrained deep networks for expressive covariance-based data representation.
    \item We study the stability of CSTs to finite-sample covariance estimates and input noise. Our results show that CSTs are less affected by close covariance eigenvalues than PCA and their error decreases as more samples $T$ are observed with rate $\mathcal{O}(T^{-1/2})$, extending the advantages of VNNs to the untrained setting. 
    \item We show that CSTs produce stable representations on 4 datasets of cortical thickness measurements which can be used by downstream linear regressors to predict patients' age with a performance that matches or beats more complex non-linear methods like VNNs with few labels.
\end{itemize}

\begin{figure}[t]
    \centering
    \includegraphics[width=1\linewidth]{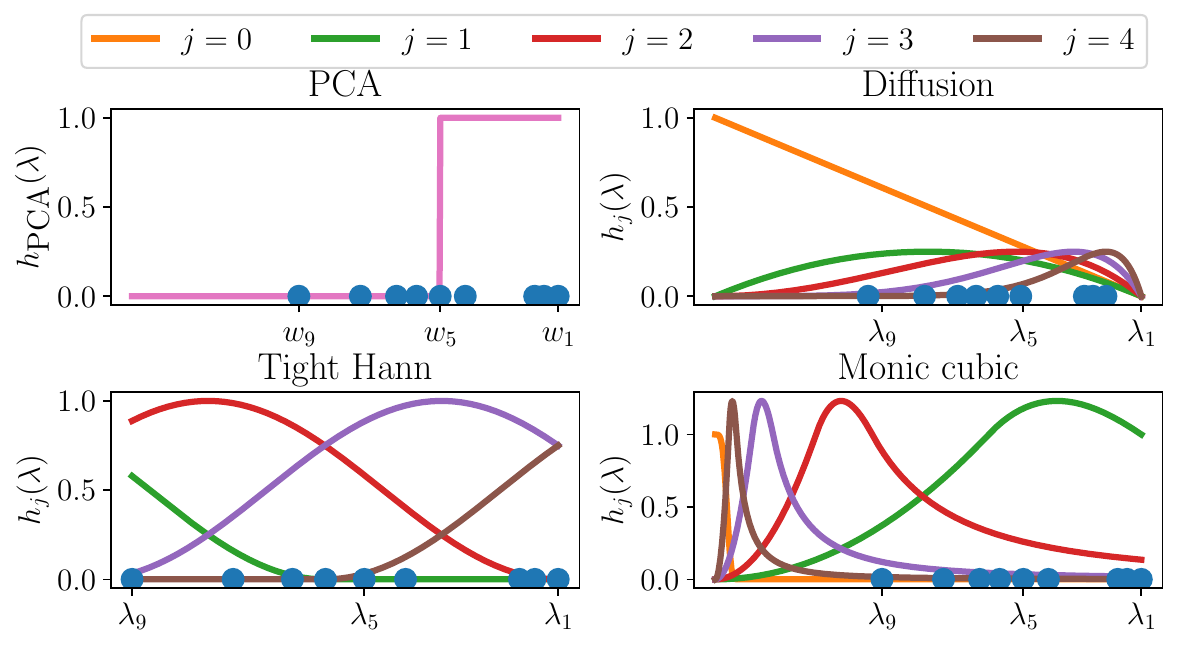}
    \caption{Filters on the covariance eigenvalues $w$ and their scaled versions $\lambda$ (see Method section for details). PCA acts as a high-pass filter selecting only the top $k$ eigenvalues (here $k=5$), whereas covariance wavelets provide more complex and localized filter shapes.}
    \label{fig:motivation_example}
\end{figure}

\section{Background}

Consider a data matrix $\mtX \in \mathbb{R}^{N\times T}$ containing $T$ observations of a random variable $\vcx \in \mathbb{R}^N$. 
The variable $\vcx$ has covariance $\mtC = \mathbb{E}[(\vcx-\mathbb{E}[\vcx])(\vcx-\mathbb{E}[\vcx])^\Tr]$, which is estimated from the $T$ samples as $\mthC = \sum_{t=1}^T (\vcx_t -\vchmu)(\vcx_t -\vchmu)^\Tr / T$ where $\vchmu = \sum_{t=1}^T \vcx_t / T$ is the sample mean. The covariance matrix admits the eigendecomposition $\mtC = \mtV\mtW\mtV^\Tr$ where $\mtV$ contains the orthogonal eigenvectors in its columns and $\mtW = \operatorname{diag}(w_1, \dots, w_N)$ contains the ordered eigenvalues $w_1 \geq w_2 \geq \dots \geq w_N$. We denote the corresponding sample estimates as $\mthC = \mthV \mthW\mthV^\Tr$.

\noindent \textbf{PCA transform.}
The PCA transform projects the data onto the eigenvectors of the covariance matrix, i.e., $\vctx = \mthV^\Tr\vcx$, where each eigenvector captures a portion of variance of the data measured by its corresponding eigenvalue.
Often, a subset with the largest $k$ eigenvectors is used to represent the data, which leads to PCA for dimensionality reduction:  $\vctx_{(k)} = [\mthV]_{1,\dots,k}^\Tr\vcx$, where $[\cdot]_{1,\dots,k}$ selects the first $k$ columns.
This can be written as $\vctx_{(k)} = [ \operatorname{diag}(h_\textnormal{PCA}(\schw_1), \dots, h_\textnormal{PCA}(\schw_N))\mthV^\Tr\vcx]_{1,\dots,k}$, where $h_\textnormal{PCA}(w) = \mathbf{1}[ w \geq \schw_k]$ is a high-pass filter in the covariance eigenvalues (see Figure~\ref{fig:motivation_example}).
While this allows for dimensionality reduction that retains high-variance information, it fails to represent information that is localized in low-variance directions, which might still be relevant for the task at hand.
Moreover, the PCA transform and its successive filtrations are unstable to covariance estimation errors. Specifically, the projection error of a data sample $\vcx$ on the true and perturbed covariance eigenvectors $\mtV, \mthV$ is bounded with high probability as~\cite[Proposition 1]{cavallo2024stvnn}:
\begin{equation}\label{eq:pca_bound}
    \| [\mthV]_{1,\dots,k}^\Tr\vcx - [\mtV]_{1,\dots,k}^\Tr\vcx \| \leq \mathcal{O}((\min_{\substack{i,j = 1, \dots, k,\\i \neq j}}|w_i - w_j|)^{-1})
\end{equation}
where $\|\cdot\|$ denotes the 2-norm for vectors and spectral norm for matrices throughout the paper.
That is, if the covariance matrix has close distinct eigenvalues $w_i \approx w_j, i\neq j$, then the principal component estimation becomes difficult and requires large amounts of observations to be reliable.

\noindent \textbf{Covariance filters.}
We provide a different interpretation for the PCA transform by building a weighted graph with $N$ nodes where the weight of the edge $(i,j)$ is the covariance value $[\mthC]_{ij}$ and the $i$-th node has as signal the $i$-th entry of the vector $\vcx$ (cf. Figure~\ref{fig:cov_graph} in the Appendix).
The graph Fourier transform of $\vcx$ is defined as its projection on the graph eigenvectors, i.e., $\vctx = \mthV^\Tr\vcx$, which coincides with the space of the PCA transform.
To increase expressivity, covariance filters~\cite{sihag2022covariance} define a general function $h(w)$ computed on each distinct covariance eigenvalue to modulate the corresponding eigenvector. This function is instantiated via a polynomial $h(w) = \sum_{k=0}^K h_k w^k$, where the coefficients $h_k$ are learned to optimize a task-specific loss.
The processing of a signal $\vcx$ via the polynomial covariance filter can be performed directly in the covariance space as $\mtH(\mthC)\vcx = \mthV \operatorname{diag}(h(\schw_1), \dots, h(\schw_N))\mthV^\Tr = \sum_{k=0}^K h_k \mthC^k\vcx$. Moreover, covariance filters can be assembled into sequential filterbanks interleaved with nonlinearities to define coVariance Neural Networks (VNNs), which learn hierarchical representations $\vcz_\ell = \sigma (\mtH_\ell(\mthC)\vcz_{\ell-1})$ with $ \vcz_0 = \vcx$, $\ell = 1, \ldots, L$.
The last layer output $\vcz_L$ represents the VNN output and is often fed to a task-specific readout function.
VNNs can learn a large class of functions in the covariance eigenspace, extending the PCA transform, and achieve better stability in low-sample regimes~\cite{sihag2022covariance}. However, they require training and labeled data to optimize the parameters $h_k$.

\noindent \textbf{Problem statement.}
Aiming to retain the expressivity and stability of VNNs as well as the untrained/unsupervised nature of the PCA transform, we aim to develop a neural architecture with the following desiderata: D1) use the sample covariance matrix as inductive bias; D2) be expressive and flexible to handle information across all the covariance spectrum; D3) be stable to finite-sample estimation errors in the covariances and input signal; D4) be untrained.

\begin{figure}[t]
    \centering
    \includegraphics[width=1\linewidth]{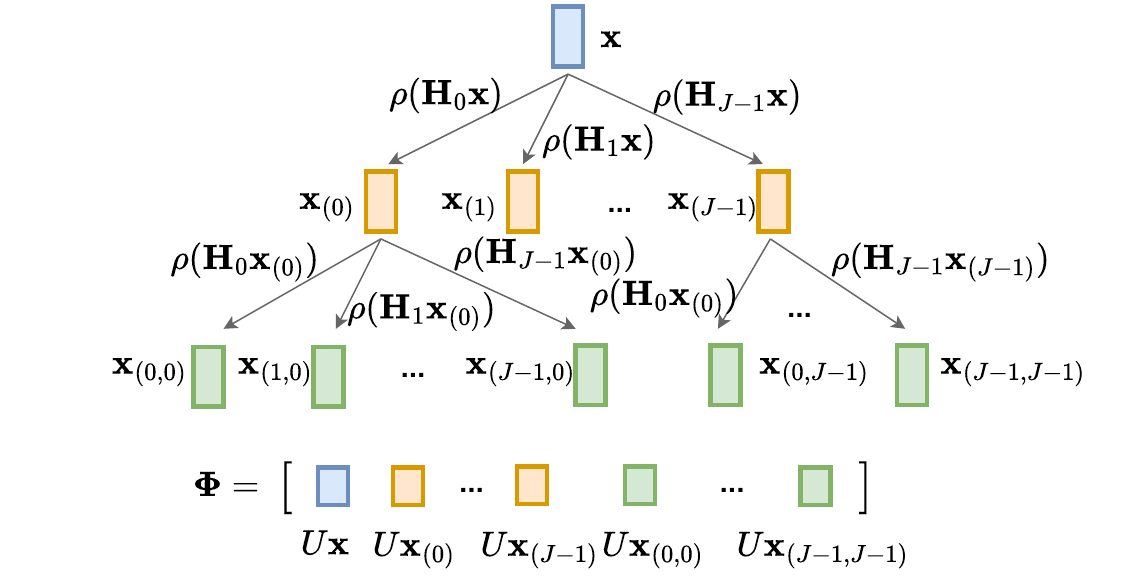}
    \caption{CSTs features are obtained via sequential application of wavelets at different scales, interleaved with non-linear activations $\rho$ and aggregation operators $U$ to produce the final coefficients collected in $\mtPhi$.}
    \label{fig:cst_diagram}
\end{figure}

\section{Method}
\label{sec:method}

We define the Covariance Scattering Transform (CST), a hierarchical untrained network that manipulates the covariance spectrum and meets our desiderata. 
The building block of CSTs are covariance wavelets, spectrally localized filters that are applied sequentially with nonlinearities in between. 
% We define them here and we propose a pruning mechanism to reduce the computational complexity.

\subsection{Covariance Wavelets}

Covariance wavelets are functions that span the spectrum of the covariance matrix of the data.
Given a covariance matrix $\mtC = \mtV\mtW\mtV^\Tr$, covariance wavelets operate on a wavelet operator matrix $\mtT = \mtV f(\mtW) \mtV^\Tr = \mtV \mtLambda \mtV^\Tr$ where $f(\cdot)$ preserves the positive semi-definiteness of $\mtT$.
In particular, we propose two implementations for $\mtT$:  the covariance matrix with normalized eigenspectrum, i.e., $\mtT = \mtC_N =  \gamma  \mtC / w_1$, where $w_1$ is the largest covariance eigenvalue and parameter $\gamma  \in \mathbb{R}_+$ controls the domain of the spectrum of $\mtT$ (i.e., the interval $[0,\gamma ]$); and $\mtT = \mtC_I =  \gamma (\mtI-\mtC / w_1)$ (where $\mtI$ is the identity matrix), which reverses the order of the covariance eigenvalues in the interval $[0,\gamma ]$, bringing benefits for wavelets that are more discriminative at lower frequencies. We denote with $\mthT = \mthV \mathbf{\hat{\mtLambda}} \mthV^\Tr$ the sample estimate of $\mtT$ from $\mthC$ and corresponding eigendecomposition. 

Covariance wavelets are characterized by a function $h_j(\lambda)$ that acts as a band-pass filter (i.e., $h_j(0)=0$ and $\lim_{\lambda \rightarrow \infty} h_j(\lambda) = 0$, cf. the wavelet admissibility criteria in~\cite{mallat1999wavelet}) and is instantiated at different \textit{scales} $j \geq 1$, i.e., different localizations in the covariance eigenvalues. 
To also account for content at $\lambda=0$, we define an additional function $h_0(\lambda)$ such that $h_0(0) > 0$. The collection of $J$ wavelet functions $\{ h_j(\lambda) \}_{j=0, \dots, J-1}$ is a multiscale filterbank that spans the covariance spectrum.
For a covariance operator $\mtT$, 
the application of the wavelet function $h_j(\lambda)$ on all its eigenvalues $\lambda_1, \dots, \lambda_N$ produces 
the wavelet matrix $\mtH_j(\mtT) = \mtV \operatorname{diag}(h_j(\lambda_1), \dots, h_j(\lambda_N))\mtV^\Tr$.
The \textit{wavelet coefficients} (or \textit{wavelet features}) of a signal $\vcx$ are its projection on the corresponding wavelet matrix, i.e., $\vcx_{(j)} = \mtH_j(\mtT) \vcx$.
We propose three implementations of covariance wavelets: diffusion, Hann and monic. From Figure~\ref{fig:motivation_example}, diffusion wavelets are more localized at high frequencies as the scales increase; Hann wavelets are localized on the specific eigenvalues; monic wavelets are sharper at low frequencies and capture higher frequencies as the scale increases. We define diffusion wavelets next, while we defer Hann and monic wavelets to the Appendix.

\noindent \textbf{Covariance diffusion wavelets.}
In analogy with graph diffusion wavelets~\cite{gama2019diffusion}, we design the covariance diffusion wavelet function as:
\begin{equation}
    h_j(\lambda) = \lambda^{2^{j-1}} - \lambda^{2^{j}}, \quad j \geq 1
\end{equation}
and $h_0(\lambda) = 1-\lambda$. Given a covariance operator $\mtT$, the graph diffusion wavelet can be computed without taking the eigendecomposition as $\mtH_j(\mtT) = \mtT^{2^{j-1}}-\mtT^{2^{j}}$. This reduces the computational complexity, as eigendecomposition is of order $\mathcal{O}(N^3)$ while the recursive computation of $\mtT^{k}\vcx = \mtT(\mtT^{k-1}\vcx)$ is of order $\mathcal{O}(kN^2)$.
To maximize the expressiveness of diffusion wavelets, given a wavelet filterbank with $J$ scales (i.e., $j=0, \dots, J-1$), we rescale the covariance eigenvalues by setting $\gamma  = \left( 1/2 \right)^{1/2^{J-2}}$, such that the $J$-th wavelet reaches its maximum on the largest covariance eigenvalue.
Since this rescaling ensures that $\lambda_1 \leq 1$, $h_j(\lambda)$ behaves as a bandpass filter as $h_j(0) = h_j(1) = 0$.

\noindent \textbf{Properties of covariance wavelets.} 
The covariance wavelets enjoy the following properties.

\begin{property}[Frame]\label{lemma:frame}
The covariance wavelets conform a frame, i.e., $A^2 \|\vcx\| \leq \sum_{j=1}^J \|\mtH_j\vcx\|^2 \leq B^2 \|\vcx\|^2 $ with $0 < A \leq B < \infty$. 
\end{property}

\begin{property}[Lipschitz]\label{lemma:lipschitz}
The covariance wavelets $h_j(\lambda)$ are Lipschitz, i.e., 
$|h_j(\lambda_k)-h_j(\lambda_l)|/|\lambda_k-\lambda_l|\leq P$ for two covariance eigenvalues $\lambda_k,\lambda_l$ and a constant $P>0$.
\end{property}

\begin{property}[Localization]\label{lemma:localization}
    Covariance wavelets are localized both in frequency and covariance space.
\end{property}

Property~\ref{lemma:frame} characterizes the spread of energy of a wavelet filterbank through the constants $A,B$.  
Property~\ref{lemma:lipschitz} describes the variability of the wavelet $h_j$ by limiting its derivative through the constant $P$.
Property~\ref{lemma:localization} extends the joint spectral and spatial localization of existing wavelets to the covariance case. Spectral localization corresponds to concentration of $h_j(\lambda)$ around a specific eigenvalue. Covariance space localization, instead, corresponds to the fact that wavelets centered on a feature assume smaller values on features that are distant from the center one, where the distance depends on the strength of the covariance among the features (see Appendix for a formal definition and analysis).
For diffusion wavelets, the frame bounds are $B=1$ and $A=1-\gamma $, i.e., using more scales $J$ leads to larger $\gamma $ and more spread-out eigenvalues, but may cause larger energy loss as $A$ becomes smaller.
The Lipschitz constant of a diffusion wavelet at scale $j$ for $\lambda \in [0,1]$ is $P=2^{j-1}$, i.e., a larger scale leads to sharper variations in the wavelet functions and consequently a larger Lipschitz constant.

\subsection{Covariance Scattering Transforms}

Covariance Scattering Transforms (CSTs) are deep architectures defined by $L$ layers and a bank of multiresolution covariance wavelets $\{ \mtH_j \}_{j=0}^{J-1}$. Given an input $\vcx$, the CST produces a set of scattering features and concatenates them.
At layer $\ell=0$, the scattering features are the input features $\vcx$.
At layer $\ell=1$, the input $\vcx$ is projected on all $J$ wavelets and processed by nonlinearity $\rho$, i.e., $\vcx_{(j_1)} = \rho(\mtH_{j_1}\vcx)$ for $j_1=0,\dots,J-1$.
This process is repeated recursively, i.e., $\vcx_{(j_\ell,j_{\ell-1}\dots,j_1)} = \rho(\mtH_{j_\ell}\vcx_{(j_{\ell-1}\dots,j_1)})$ for $j_\ell=0,\dots,J-1$ at every layer $\ell=1,\dots,L-1$, creating a hierarchical structure of scattering representations (cf. Figure~\ref{fig:cst_diagram} and Algorithm 1 in the Appendix).
By expanding the recursion, the scattering features at the $\ell$-th layer are
\begin{equation}
    \vcx_{(j_\ell,j_{\ell-1},\dots ,j_1)} = \rho (\mtH_{j_\ell} \rho (\mtH_{j_{\ell-1}}\dots \rho (\mtH_{j_1}\vcx)) ).
\end{equation}
The scattering features $\vcx_{(j_\ell,j_{\ell-1},\dots ,j_1)}$ can be further processed by an operator $U$ (e.g., mean for dimensionality reduction, identity for no processing), and the resulting $(J^L-1)/(J-1)$ coefficients $\phi_{j_\ell,j_{\ell-1}\dots j_1}(\vcx) = U\vcx_{(j_\ell,j_{\ell-1},\dots ,j_1)}$ computed at each layer $\ell = 0, \dots, L-1$ and for all scales $j_\ell = 0, \dots, J-1$ are concatenated to define the CST $\mtPhi(\mtT,\vcx)$. 
% Figure~\ref{fig:cst_diagram} shows a representation of the hierarchical structure of the CST.

\subsection{Pruning}

CSTs' number of coefficients grows exponentially with the increasing scales and number of layers, leading to large-dimensional representations to capture deep encodings. To counteract this, we consider a pruning strategy to sparsify the CST tree and only explore the branches that have larger potential to provide meaningful representations. 
Following~\cite{ioannidis2020pruned}, given a the representation at the $\ell$-th scattering transform layer $
\vcx_{(j_\ell,\dots,j_1)}$, its projection on the $i$-th wavelet $
\vcx_{(i,j_\ell,\dots,j_1)}$ (and the subsequent projections at deeper layers) is discarded if its normalized energy is lower than a predefined threshold $\tau$, i.e., $\|\vcx_{(i,j_\ell,\dots,j_1)}\|/\|\vcx_{(j_\ell,\dots,j_1)}\| \leq \tau$. We denote with $F_\ell$ the number of scattering features selected at layer $\ell$ out of the $J^\ell$ available.
This pruning strategy reduces the search space, leading to more parameter-efficient representations and making the CST a feasible technique for dimensionality reduction. Moreover, the amount of pruning affects the stability of CSTs as we shall elaborate in Theorem~\ref{th:cst_stab}.

\section{Theoretical Analysis}

We characterize theoretically the CST by studying its permutation equivariance and its stability to perturbations in the covariance matrix. We discuss the stability to signal perturbations in the Appendix.

\subsection{Permutation Equivariance}

Since covariance information captures pairwise relations among data features that do not depend on the ordering in which features are observed, it is of interest that the CST is \textit{permutation equivariant}, i.e., if the order of its input is permuted, its output is permuted likewise. Furthermore, if the CST's representations are used to produce a unique label or regression target that does not depend on feature ordering, it is desirable that the CST output is \textit{permutation invariant}, i.e., its output does not change regardless of the input order. 
We first define the permutation operation for a CST.
\begin{definition}\label{def:cst_perm}
    Consider a signal $\vcx \in \mathbb{R}^N$ and a CST $\mtPhi = [ U\vcx || U\vcx_{(1)} || \dots || U\vcx_{(J,\dots,J)}]$, where $||$ is the concatenation operation and $U$ preserves the dimension $N$. Let $\mtPi$ be a permutation matrix. The CST permutation operator is 
    \begin{equation}
        \operatorname{Perm}(\mtPhi, \mtPi) = [ \mtPi U\vcx || \mtPi U\vcx_{(1)} || \dots || \mtPi U\vcx_{(J,\dots,J)}].
    \end{equation}
\end{definition}

The following theorem shows that the CST is permutation equivariant, and can be made permutation invariant by an appropriate choice of aggregation function $U$.

\begin{theorem}\label{th:perm_equivariance}
    Consider a CST $\mtPhi$ computed from a dataset $\mtX \in \mathbb{R}^{N\times T}$, and a CST $\mathbf{\hat{\mtPhi}}$ computed from a dataset $\mthX = \mtPi\mtX$ given permutation matrix $\mtPi \in \mathbb{R}^{N\times N}$.
    If $U$ is permutation equivariant (e.g., identity), then $\hat{\mtPhi} = \operatorname{Perm}(\mathbf{\mtPhi}, \mtPi)$. If $U$ is permutation invariant (e.g., average), then $\mtPhi = \mathbf{\hat{\mtPhi}}$.
\end{theorem}

Theorem~\ref{th:perm_equivariance} establishes that CSTs provide permutation equivariant or invariant representations depending on the design of $U$, which extends analogous results in other scattering transforms~\cite{gama2019stabilityscatter,bruna2013invariant} and respects the domain requirements.

% \subsection{Stability}
% We analyze the theoretical properties of the CST by considering its output stability in the presence of (i) covariance estimation errors due to finite-sample and (ii) input signal perturbations.

\subsection{Stability to Covariance Perturbations}

In practice, the CST is instantiated on a sample covariance $\mthC$ that represents a perturbed version of the true one, i.e., $\mthC = \mtC + \mtE_C$ where $\mtE_C$ collects the estimation error. We study here how the output of the CST changes in relation to this finite-sample error. 
We begin by bounding the output difference of covariance wavelets, i.e., given the wavelet $\mtH_j(\cdot)$ instantiated on true and sample operators $\mtT$ and $\mthT$, we are interested in the quantity 
\begin{equation}\label{eq:stab_def}
\| \mtH(\mtT) - \mtH(\mthT) \| = \min\{ c \geq 0 : \| \mtH(\mtT)\vcx - \mtH(\mthT)\vcx \| \leq c\|\vcx\| \}. 
\end{equation}
Our analysis requires the following assumptions.
\begin{assumption}\cite[Theorem 5.6.1]{vershynin2018high}
\label{as_norm}
    Given a random variable $\vcx$ and constants $G\ge 1$, $\delta \approx 0$, it holds:
    \begin{equation}\nonumber
        \mathbb{P} \left( \|\vcx\|\leq G\sqrt{\mathbb{E}[\|\vcx\|^2]}\right) \ge 1 - \delta.
    \end{equation}
\end{assumption}
\begin{assumption}\cite[Theorem 4.1]{loukas2017howclose}
\label{as_eig_diff}
    The eigenvalues $\{w_i\}_{i=1}^{N}$ and $\{\schw_i\}_{i=1}^{N}$ of the true and sample covariance matrix, respectively, satisfy for each pair $(w_i,w_j)$, $i\neq j$,
    \begin{align}\nonumber
        \textnormal{sign}(w_{i} - w_j)2\schw_{i} > \textnormal{sign}(w_{i} - w_{j})(w_{i} + w_{j}).
    \end{align}
\end{assumption}

Assumption~\ref{as_norm} quantifies the variance of the data distribution via the constant $G$, which is higher for data with higher variance. 
Assumption~\ref{as_eig_diff} considers the estimation error of sample covariance eigenvalues compared to the true ones, and holds for each eigenvalue pair $(w_i,w_j)$ with probability at least $1-2k_{i}^2/(N|w_{i} - w_{j}|)$, where $k_{i}=\left( \mathbb{E}[\|\vcx\vcx^\Tr\vcv_{i}\|^2]-w_{i}^2 \right)^{1/2}$ is a term related to the kurtosis of the data distribution~\cite[Corollary 4.2]{loukas2017howclose}.

We now provide the covariance wavelet stability result.
\begin{theorem}\label{lemma:wavelet_stability}
    Consider a covariance wavelet $\mtH_j(\cdot)$ with Lipschitz constant $P$ and let Assumptions 1,2 hold. The output difference of the wavelet computed on the true and perturbed covariance wavelet operators $\mtT$ and $\mthT$, respectively, is bounded with probability at least $(1-e^{-\epsilon})(1-2e^{-u})$ as 
    \begin{align}\label{eq:cov_wl_bound}
        \| \mtH_j(\mthT) - \mtH_j(\mtT) \| \leq \nonumber \\
        \frac{PN}{\sqrt{T}}(k_\textnormal{max}e^{\frac{\epsilon}{2}} + \frac{2QG\gamma \|\mtC\|}{w_1}\sqrt{\log{N} + u} ) + \mathcal{O}\left(\frac{1}{T}\right) := \Delta \nonumber
    \end{align}
    where $Q$ is an absolute constant, $k_{\textnormal{max}} = \max_jk_j$ with $k_j=\left( \mathbb{E}[\|\vcx\vcx^\Tr\vcv_j\|^2]-w_j^2 \right)^{1/2}$ related to the kurtosis of the data distribution, and $\epsilon, u >0$ are arbitrarily large constants. 
\end{theorem}
This result provides three main insights. First, the bound increases with the sample dimension $N$, as larger covariance matrices are more difficult to estimate, and decreases with the number of observed samples $T$ with the rate $\mathcal{O}(T^{-1/2})$, as the sample covariance gets closer to the true one the more samples are available. This is in contrast with the stability bound of graph wavelets~\cite{gama2019stabilityscatter}, where the graph perturbation does not reduce with the number of samples. 
Second, compared to the PCA bound in~\eqref{eq:pca_bound}, the covariance wavelet stability does not depend on the covariance eigengap, which is absorbed by the Lipschitz constant $P$. This is crucial because the constant $P$ bounds the variability of the wavelet function $h_j(\lambda)$ w.r.t. the covariance eigenvalues, such that a smaller $P$ corresponds to more slowly-varying $h_j(\lambda)$ and, consequently, better stability at the cost of lower discriminability for close eigenvalues, as common in graph and covariance neural networks~\cite{gama2020stability,sihag2022covariance}. 
The constant $P$ depends on the wavelet definition and its parameters. For diffusion wavelets, $P$ increases as the scale $j$ increases, since a larger scale introduces sharper transitions in $h_j(\lambda)$ (see the Appendix for details on $P$ for Hann and monic wavelets).
Third, the bound is modulated by the largest covariance eigenvalue $w_1$ and the parameter $\gamma$ due to the covariance normalization. This is beneficial when $w_1>1$ and $\gamma < 1$, as the bound gets lower and errors are reduced. 

After establishing the stability of covariance wavelets, we proceed to provide a condition under which the pruned branches of the CST are the same on the perturbed and true covariance in the following proposition.
\begin{proposition}\label{lemma:prune_cond}
    Consider a CST $\mtPhi$ instantiated on true and sample covariance wavelet operators $\mtT$ and $\mthT$, respectively, with $T$ observations. Let the covariance wavelet $\mtH_j$ be Lipschitz with constant $P$ and form a frame with bound $B$. The pruned trees with threshold $\tau$ of $\mtPhi(\mtT, \vcx)$ and $\mtPhi(\mthT, \vcx)$ for a generic signal $\vcx$ are identical if, for all representations $\vcx_{(j_\ell,\dots,j_1)}$ at layer $\ell$ and $j = 0, \dots, J-1$, it holds:
    \begin{align}
        | \|\mtH_j(\mtT)\vcx_{(j_\ell,\dots,j_1)}\|^2 - \tau\|\vcx_{(j_\ell,\dots,j_1)}\|^2 | > \nonumber \\
        (\Delta B^{\ell-1} \|\vcx\|)^2 ( (\ell+1)B + \ell\tau ). \nonumber
    \end{align}
    where $\Delta$ is defined in Theorem~\ref{lemma:wavelet_stability}.
\end{proposition}
This condition depends on the design of the CST via $\tau$, $B$, $P$ and the wavelet filter $\mtH_j(\cdot)$, on data characteristics via $\mtC$ and $N$ and on the number of observed samples $T$. In particular, the condition becomes more likely as $T$ increases, since the covariance estimation improves and the perturbation affects the pruning less. 
Moreover, the condition becomes more likely for smaller $\tau$ as the left-hand term increases while the right-hand term decreases. 

With this in place, we investigate the stability to covariance estimation errors of CST.
\begin{theorem}\label{th:cst_stab}
Consider a CST $\mtPhi(\cdot)$ with $L$ layers and $J$ scales operating on a true covariance operator $\mtT$ and sample operator $\mthT$ estimated from $T$ samples. Let $\Delta$ and $B$ be the largest stability and frame bounds, respectively, among the wavelets in the CST, and $\|U\|\leq B_U$, and let the condition in Proposition~\ref{lemma:prune_cond} hold. 
The distance between the CST representations operating on the true and estimated operators $\mtT$, $\mthT$ can be upper-bounded as
\begin{align}
    \|\mtPhi(\mtT,\vcx) - \mtPhi(\mthT,\vcx)\| \leq B_U\Delta\|\vcx\|\sqrt{\sum_{\ell=1}^{L-1} \ell^2 B^{2\ell-2}F_\ell }. \nonumber
    \end{align}    
where $F_\ell \leq J^\ell$ is the number of selected scattering features at layer $\ell$. 
\end{theorem}

Theorem~\ref{th:cst_stab} proves that CSTs are stable to covariance perturbations proportionally to the stability of a single wavelet $\Delta$. The stability bound increases with the total number of active scattering features (such that more pruning leads to better stability), number of layers $L$ and frame bound $B$.
$F_\ell $ increases with increasing $J$ and decreasing $\tau$ (e.g., $F_\ell = J^\ell$ for $\tau=0$), making the CST less stable while improving its expressivity via more coefficients at different scales. A proper choice of pruning threshold $\tau$ can help in this tradeoff by allowing for wavelets at larger scales $j$ that carry relevant information, while pruning less informative wavelets at smaller scales.
The number of scales $J$ also increases the wavelet bound $\Delta$ for diffusion wavelets, as more scales lead to sharper wavelet transitions and larger $P$. 
The CST compensates the reduced expressivity of a single wavelet by the cascade of wavelet filterbanks interleaved with nonlinearities, which spread information across the frequency spectrum and increase the discriminability at deeper layers, reiterating the advantage of using deep architectures~\cite{isufi2024graphfilters}. 
Compared to VNNs~\cite[Theorems 1-2]{sihag2022covariance}, this bound can be made smaller via a larger $\tau$, whereas VNNs do not have any pruning mechanism. Furthermore, the Lipschitz constant of VNNs depends on the training dynamics, whereas CSTs are untrained and their Lipschitz constant depends on the choice of wavelet functions and scales, thus it can be defined a priori. 
VNNs also achieve a tradeoff between stability and expressivity thanks to their hierarchical deep architecture, but they require labeled data for parameter training while CSTs achieve the same result without any training.

\begin{figure*}[t]
    \centering
    \includegraphics[width=1\linewidth]{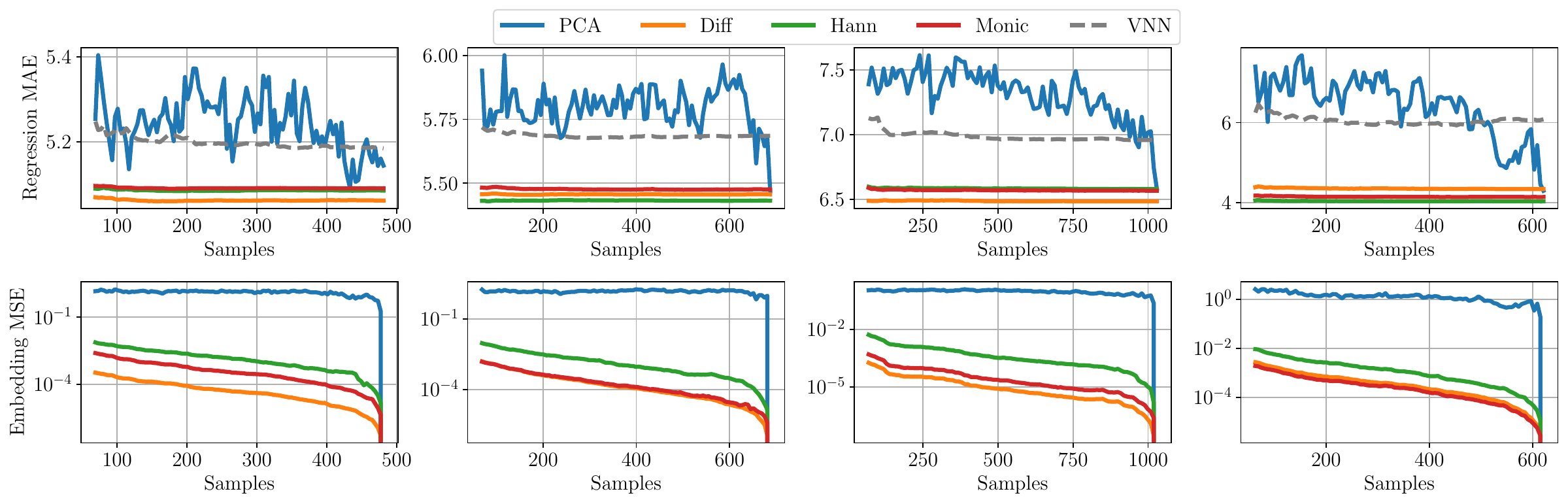}
    \caption{Age prediction Mean Average Error (MAE) and embedding Mean Squared Error (MSE) for increasing number of samples for CSTs, VNN and PCA on, from left to right, ADNI1, ADNI2, PPMI and Abide.  }
    \label{fig:stab_real}        
\end{figure*}

\section{Numerical Results}

We evaluate CSTs with the following objectives: \textbf{(O1)} validate their stability to covariance estimation errors; \textbf{(O2)} show the effectiveness of their representations for downstream tasks; \textbf{(O3)} assess the impact of pruning on performance and time and parameter efficiency.

\subsection{Setup}
\noindent \textbf{Datasets.}
We consider four datasets containing cortical thickness measurements extracted from MRI scans for patients with a specific disease and healthy control patients: \textbf{ADNI1} and \textbf{ADNI2}~\cite{jack2008alzheimer}, \textbf{PPMI}~\cite{marek2011parkinson} and \textbf{Abide}~\cite{craddock2013neuro} (details in Table~\ref{tab:datasets}). 
Cortical thickness is the thickness of the cerebral cortex in various regions of brain and is correlated to a patient's age, as it tends to thin when patients get older.
We model the brain areas as $N$ nodes and the thickness measures as node signals. Each of the $T$ patients is an observation and the covariance among thickness measures constitutes our graph.
Given this setup, our downstream regression task is to predict the patient's age, which is of high interest to identify neurodegenerative diseases or accelerated brain aging corresponding to a gap between predicted brain age and chronological age~\cite{sihag2024explainable,bashyam2020mri,yin2023anatomically}. 

\noindent \textbf{Models and baselines.}
We compare the 3 proposed implementations of CSTs (diffusion, Hann and monic) to 
(i) PCA, which processes covariance information in an unsupervised manner, but lacks stability and expressivity;
(ii) VNN~\cite{sihag2022covariance}, which achieves stability and expressive covariance manipulation via supervised training;
(iii) a ridge regressor on the raw features, which does not consider covariance information.
For the downstream task, we feed the representations of CSTs and PCA to a ridge regressor.
We optimize all hyperparameters via grid search on a validation set. We report the grids and the final choices in Table~\ref{tab:params} in the Appendix. We repeat every experiment over 10 different splits.
We report additional experiments on controlled setups, more baselines and ablations on the real datasets in the Appendix.
For all regression tasks, standard deviations are of the order $10^{-1}$. We do not plot them for visual clarity, but we report them in Table~\ref{tab:extra_comparisons} in the Appendix. 

\begin{table}[t]
\centering
\small
\begin{tabular}{c|cccc}
\toprule
\textbf{} & \textbf{ADNI1} & \textbf{ADNI2} & \textbf{PPMI} & \textbf{Abide} \\
\midrule
\textbf{Patients} ($T$) & 801 & 1142 & 1704 & 1035 \\
\textbf{Brain areas} ($N$) & 68 & 68 & 68 & 62 \\
\bottomrule
\end{tabular}
\caption{Dataset characteristics.}
\label{tab:datasets}
\end{table}

\begin{figure*}[t]
    \centering
    \includegraphics[width=1\linewidth]{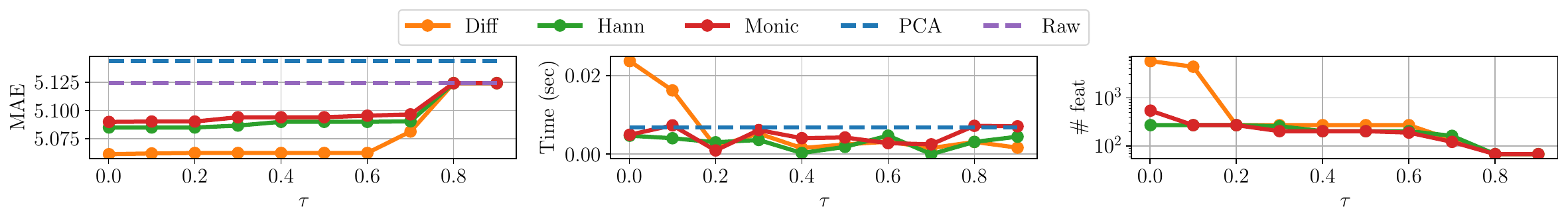}
    \caption{Impact of different thresholds $\tau$ for pruning on regression MAE, execution time and number of features on ADNI1. }
    \label{fig:pruning_adni1}        
\end{figure*}

\subsection{Stability and Regression Performance}

\noindent \textbf{Experimental setup.}
We keep $50\%$ of the data as unlabeled (i.e., we do not use the age information of these patients for the downstream task), and we split the remaining as $10\%$ for training, $20\%$ for validation and $20\%$ for testing. Let $\mathcal{U}$ be the union of the unlabeled and training sets.
We compute the CST on $\mathcal{U}$ and we use it to produce representations for the test set, which we feed to a ridge regressor for the downstream task. 
To investigate the role of finite-sample covariance estimation errors, we recompute the CST using a covariance estimated from a subset of $\mathcal{U}$. With this perturbed CST, we produce new representations for the test samples, which we feed to the previously trained regressor for the downstream task. We do not perform thresholding.

\noindent \textbf{Discussion.}
Figure~\ref{fig:stab_real} shows that CSTs maintain a consistent regression performance under covariance perturbations, demonstrating the empirical advantages of their stability \textbf{(O1)}. This significantly improves over PCA-based preprocessing, which suffers large instabilities in low-sample regimes, ultimately leading to significantly worse performance. The same observations stem from the embedding MSE (i.e., the quantity bounded in Theorem~\ref{th:cst_stab}), which is contained for CSTs, while it grows larger for PCA.
Moreover, CSTs' representations generally achieve better performance than PCA, VNN and raw features, demonstrating the informativeness of such embeddings and the usefulness of spectrally-localized information \textbf{(O2)}.
VNN, while stable to covariance perturbations, achieves overall worse performance than CSTs, likely due to its higher training complexity which requires larger quantities of labeled data, whereas CSTs can exploit the unlabeled samples to produce meaningful representations.

\begin{figure}[t]
    \centering
    \includegraphics[width=1\linewidth]{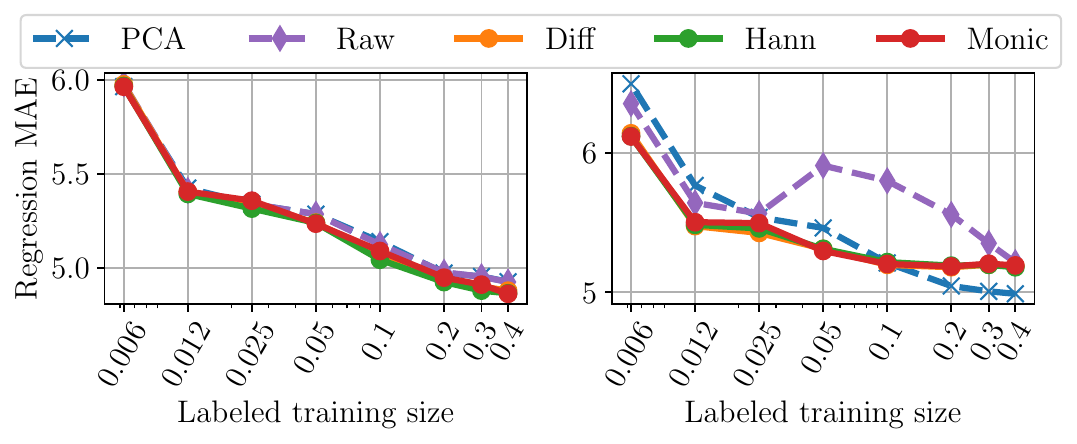}
    \caption{MAE for different labeled data sizes on ADNI1 for $U$ as identity operator (left) and $U$ as mean operator (right).}
    \label{fig:lab_size_adni1}        
\end{figure}

\subsection{Additional Results}

We assess the impact of pruning, labeled data size and aggregation on ADNI1. The results for the other datasets are presented in the Appendix. We keep the same hyperparameter configuration as in Figure~\ref{fig:stab_real} except for the aggregation experiments where we re-optimize parameters on splits with $59.4\%$ unlabeled - $0.6\%$ train - $20\%$ valid - $20\%$ test sizes.

\noindent \textbf{Impact of pruning.}
We evaluate the impact of pruning on model performance, time efficiency and number of selected features \textbf{(O3)}. Figure~\ref{fig:pruning_adni1} shows that increasing $\tau$ does not lead to drastic changes in regression MAE, whereas it decreases more significantly the number of features and execution time, leading to faster computations for CSTs compared to PCA. 
Therefore, the adopted pruning approach can significantly impact the time and memory requirements of the CST, while guaranteeing expressive data representations. 

\noindent \textbf{Impact of size of labeled dataset.}
We evaluate the CST's performance when the size of labeled data for training varies from $0.6\%$ to $40\%$ of the total dataset, while validation and test sets remain fixed to $20\%$ each and the remaining portion of data is unlabeled. 
Figure~\ref{fig:lab_size_adni1} (left) shows that CSTs perform similarly or slightly better than PCA and raw features for all labeled data sizes. This corroborates the capability of CSTs to capture relevant information from unlabeled data, which is not exploited by pure learning approaches. 

\noindent \textbf{Dimensionality reduction.}
Finally, we evaluate the CST's performance when aggregating the scattering features via a mean operation for varying size of training data. 
Figure~\ref{fig:lab_size_adni1} (right) shows that reducing the data dimensions leads to advantages for low-training-data settings, where CSTs outperform raw features and PCA. For larger training data size, however, PCA leads to better results, as the regressor can effectively exploit its extensive information which is lost during the averaging aggregation of CST.

\section{Related Works}

\noindent \textbf{Covariance-based learning.}
Covariance information is at the basis of several data processing techniques. In the unsupervised domain, PCA~\cite{Jolliffe2002pca} and factor analysis~\cite{child2006essentials}, among others, are popular as they represent data in a low-rank space, where spurious correlations are removed and data dimension can be reduced.
However, PCA is generally unstable to finite-sample covariance estimation errors, leading to unreliable component estimation in low-data regimes. This issue has been tackled by coVariance Neural Networks (VNNs)~\cite{sihag2022covariance}, which rely on labeled data to estimate robust representations even in low-data regimes. This characteristic has made VNNs successful in a variety of settings, ranging from interpretable brain age estimation~\cite{sihag2024explainable, sihag2022covariance} to temporal data~\cite{cavallo2024stvnn}, sparse covariances~\cite{cavallo2024sparsecovarianceneuralnetworks,cavallo2025precision} and biased datasets~\cite{cavallo2025fair}.
Despite their success, VNNs and extensions need large quantities of labeled data for training, which may be unfeasible in practice. 
In this work, we provide a more flexible and robust framework to process data via their covariance information in an untrained manner.

\noindent \textbf{Wavelets and scattering transforms.} 
Wavelet transforms are popular tools in time, image and graph signal processing due to their space and frequency localization that allows for efficient signal representation and processing~\cite{mallat1999wavelet, hammond2011wavelets, shuman2015spectrum}. 
Scattering transforms are cascades of wavelets interleaved with nonlinearities that achieve untrained deep hierarchical representations, and have been shown successful in a variety of domains ranging from images~\cite{bruna2013invariant} to graphs~\cite{gama2019diffusion,gama2020stability,koke2022graph}, audio~\cite{anden2011multiscale} and simplicial complexes~\cite{madhu2024unsupervised}. Their main advantages are their stability to domain and signal perturbations as well as their capability to extract expressive frequency patterns in an untrained way.
In this work, we build on this literature to propose scattering transforms on covariance matrices, study their link with PCA and VNNs via covariance spectrum processing and their increased stability to finite-sample estimation errors.

\section{Conclusions}

We introduced Covariance Scattering Transforms (CSTs), untrained hierarchical deep networks that generate expressive representations for data by manipulating their covariance matrix. We proved that CSTs are permutation equivariant, achieve stability to covariance perturbations and produce rich data embeddings that lead to good performance on an age prediction task from cortical thickness measurements in four different datasets.
CSTs present a tradeoff between increasing the sample dimension for improved expressivity or reducing it via low-pass aggregations that lose information. Future work will address this aspect and propose more effective aggregation functions. Moreover, the application of this framework to other settings where covariance information plays a crucial role, such as financial data and sensor measurements, represents a promising extension.

\section{Acknowledgements}
Part of this work was funded by the TU Delft AI Labs program, the NWO OTP GraSPA proposal \#19497, the NWO VENI proposal 222.032.

%%%%%%%%%%%%%%%%%%%%%%%%%%%%%%%%%%%%%%%%%%%%%%%%%%%%%%%%%%%%

\bibliography{bibliography}

\newpage 

\appendix
\clearpage

\section{Covariance Graph and VNNs}

We provide additional details on the graph built from the covariance matrix which is useful to interpret the setting of coVariance Neural Networks (VNNs)~\cite{sihag2022covariance} and covariance wavelets. 
Given a sample $\vcx = [x_1, x_2, \dots, x_N]^\Tr \in \mathbb{R}^N$ with covariance $\mtC$, consider a weighted graph with $N$ nodes where each node $i$ has a signal $x_i$ corresponding to the $i$-th entry of $\vcx$ and the edge between nodes $i$ and $j$ is given by the covariance value among the corresponding features $c_{ij}$ (see Figure~\ref{fig:cov_graph} for a visualization).
The covariance graph is a useful interpretation at the basis of VNNs, which perform graph convolutions on it. 
Specifically, a covariance filter performs the operation
\begin{equation}
    \vcz = \mtH(\mtC)\vcx = \sum_{k=0}^K h_k \mtC^k \vcx
\end{equation}
where $K$ is the filter order and $h_k$ are learnable parameters.
A VNN layer consists of a filterbank of $F_{l-1}\times F_{l}$ covariance filters and a nonlinearity $\sigma(\cdot)$, i.e., 
\begin{align}
\label{eq:vnn_layer}
     \vcz^l_f = \sigma\left(\sum_{g = 1}^{F_{l-1}}\mtH^l_{fg}(\mtC)\vcz^{l-1}_g\right)~f=1,\ldots,F_{l},~l = 1, \ldots, L
\end{align}
and produces outputs $\{\vcz^{l}_f\in\mathbb{R}^{N}\}_{f=1}^{F_l}$.
The $f$-th covariance filter bank contains $F_{l-1}$ covariance filters $\{ \mtH_{fg}^l(\mthC) \}_{g=1}^{F_{l-1}}$ that separately process each of the signals generated at the previous layer $\{\mathbf{z}_g^{l-1}\in\mathbb{R}^{N}\}_{g=1}^{F_{l-1}}$, generating the vector $\vcz_f^l$.
The input is $\{\mathbf{z}_g^0 = \vcx_g\}_{g=1}^{F_0}$ where $F_0$ is the node feature size, and the output of the last layer $\mathbf{z}^L$ contains the final representations generated by the VNN, which are generally fed to a readout layer for the final task.

The connection between covariance filters and PCA becomes apparent via their spectral analysis.
Consider the graph Fourier transform of a filter
\begin{align}
    \mtV^\Tr \vcz &= \mtV^\Tr \mtH(\mtC)\vcx = \mtV^\Tr \sum_{k=0}^K h_k \mtC^k \vcx \\
    &= \mtV^\Tr \sum_{k=0}^K h_k \mtV \mtW^k \mtV^\Tr \vcx \\
    &= \sum_{k=0}^K h_k \mtW^k \mtV^\Tr \vcx.
\end{align}
where $\mtC = \mtV\mtW\mtV^\Tr$ is the covariance eigendecomposition.
That is, a covariance filter learns to scale the principal components of $\vcx$, i.e., $\mtV^\Tr\vcx$, via polynomial coefficients $h_k \mtW^k$.

% \section{Introduction to Wavelets}

\begin{figure}
    \centering
    \includegraphics[width=.7\linewidth]{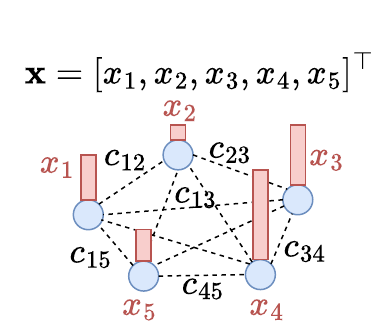}
    \caption{Covariance graph. 
    Each feature in a signal $\vcx$ becomes a node in a graph, the corresponding value becomes a node signal (i.e., one value per node, the red bars) and the edges are the covariance values.}
    \label{fig:cov_graph}
\end{figure}

\section{Covariance Wavelet Definitions and Properties}

We report here the definitions for diffusion, monic cubic and tight Hann wavelets with details on their frame bounds and Lipschitz constants. We refer to Figure~\ref{fig:motivation_example} for a visual comparison among the different wavelets. 
We begin by stating a Lemma that extends~\cite[Theorem 5.8]{hammond2011wavelets} to covariance wavelets and will be useful to analyze the properties of the specific wavelets.

\begin{lemma}\label{lemma:cov_bounds}
    The frame bounds $A$ and $B$ for the covariance wavelet depend on the eigenvalues as:
\begin{align*}
    &A^2 = \min_{i=1, \dots, N} G(\lambda_i)\\
    &B^2 = \max_{i=1, \dots, N} G(\lambda_i)\\
    &G(\lambda) = \sum_{j=0}^{J-1} h_j^2(\lambda).
\end{align*}
\end{lemma}
\begin{proof}
    Consider the quantity $\sum_{j=0}^{J-1} \|\vcx_{(j)}\|^2$. We can expand it as 
    \begin{align}
        \sum_{j=0}^{J-1} \|\vcx_{(j)}\|^2 &= \sum_{j=0}^{J-1} \| \mtH_j(\mtT) \vcx\|^2 \\
        &= \sum_{j=0}^{J-1} \| \mtV h_j(\mtLambda)\mtV^\Tr \vcx \| ^ 2 \\
        & = \sum_{j=0}^{J-1}  (\mtV h_j(\mtLambda)\mtV^\Tr \vcx)^\Tr (\mtV h_j(\mtLambda)\mtV^\Tr \vcx) \\
        & = \sum_{j=0}^{J-1} \vcx^\Tr\mtV h_j(\mtLambda) \mtV^\Tr\mtV h_j(\mtLambda)\mtV^\Tr \vcx \\
        & = \sum_{j=0}^{J-1} \vcx^\Tr\mtV h_j^2(\mtLambda)\mtV^\Tr \vcx \\
        & = \sum_{j=0}^{J-1} \| h_j(\mtLambda)\mtV^\Tr  \vcx \|^2 \\
        & = \sum_{j=0}^{J-1} \sum_{i=1}^{N} h^2_j(\lambda_i) [\mtV^\Tr  \vcx]_i^2
    \end{align}
    We then get
    \begin{align}
\sum_{j=0}^{J-1} \sum_{i=1}^N h^2_j(\lambda_i)[\mtV^\Tr\vcx]_i^2 & \geq       \min_{i=1, \dots, N} \left(\sum_{j=0}^{J-1} h^2_j(\lambda_i)\right)  \sum_{i=1}^N [\mtV^\Tr\vcx]_i^2 \\
& =  \min_{i=1, \dots, N} \left(\sum_{j=0}^{J-1} h^2_j(\lambda_i)\right) \|\mtV^\Tr\vcx\|^2 \\
& =  A^2 \|\vcx\|^2,
    \end{align}
where we used the property that $\|\mtV^\Tr\vcx\| = \| \vcx \|$ since $\mtV$ is orthonormal and we defined $G(\lambda_i) = \sum_{j=0}^{J-1} h^2_j(\lambda_i)$.
An analogous derivation holds for $B^2$.
\end{proof}

We also state a property that is useful in the following.

\begin{property}[Cayley-Hamilton]\label{lemma:cov_wl}
A covariance wavelet $h_j(\lambda)$  instantiated on $N$ eigenvalues can be written as a polynomial of order at most $N-1$ (cf. Cayley-Hamilton theorem~\cite[Theorem 2.4.2]{horn2012matrix}). That is, the covariance wavelet can be written as a convolutional covariance filter $\mtH_j(\mtT) = \sum_{k=0}^K h_k \mtT^k $ with $K=N-1$. 
\end{property}
Specifically, Property~\ref{lemma:cov_wl} allows to analyze the properties of a wavelet via its equivalent covariance convolutional filter, which we use in the proofs. 

We now proceed to define and characterize the three covariance wavelets. 

\subsection{Covariance Diffusion Wavelets}

Given a covariance $\mtC$ and its corresponding wavelet operator with eigendecomposition $\mtT = \mtV\mtLambda\mtV^\Tr$ parametrized by $\gamma$ (cf. wavelet definition in main text), we define the covariance diffusion wavelet function as
\begin{equation}
    h_j(\lambda) = \lambda^{2^{j-1}} - \lambda^{2^{j}}
\end{equation}
and $h_0(\lambda) = 1-\lambda$. Given a filterbank with $J$ scales, we set $\gamma  = \left( 1/2 \right)^{1/2^{J-2}}$ such that the $J$-th wavelet is centered on the largest eigenvalue of $\mtT$.

\begin{figure*}[t]
    \centering
    \includegraphics[width=\linewidth]{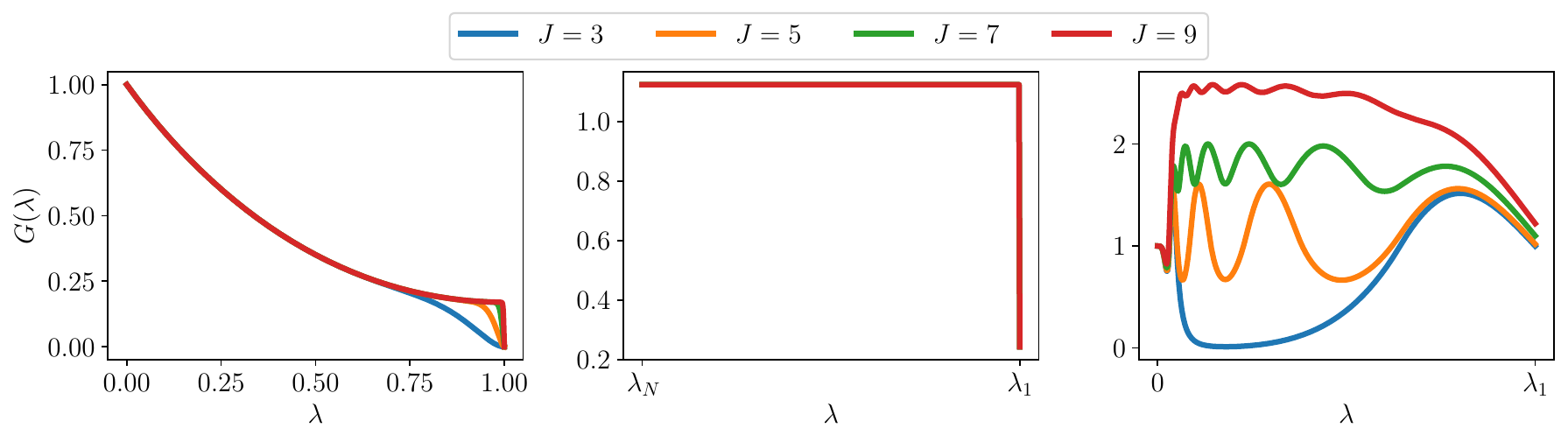}
    \caption{$G(\lambda)$ for diffusion (left), Hann (center) and monic (right) wavelets.} 
    \label{fig:g_lambda}        
\end{figure*}

\noindent \textbf{Frame bounds.} 
To compute the frame bounds $A,B$ for the covariance diffusion wavelets, we apply Lemma~\ref{lemma:cov_bounds} with $G(\lambda) = h_0^2(\lambda) + \sum_{j=1}^{J-1} h_j^2(\lambda) = (1-\lambda)^2 + \sum_{j=1}^{J-1} (\lambda^{2^{j-1}} - \lambda^{2^{j}})$.
Since $G(\lambda) \leq 1$ for $\lambda \in [0,1]$ (cf. Figure~\ref{fig:g_lambda} (left)),  $B$ can be picked to be 1.
Moreover, $G(\lambda) \geq (1-\lambda)^2$ since $\sum_{j=1}^{J-1} (\lambda^{2^{j-1}} - \lambda^{2^{j}})$ is positive in the domain and therefore $A^2$ can be picked to be $(1-\lambda_1)^2$, i.e., it depends on the largest eigenvalue of $\mtT$. For both $\mtT = \mtC_N$ and $\mtT = \mtC_I$, $\lambda_1 = \gamma  = \left( 1/2 \right)^{1/2^{J-2}}$ due to rescaling.

\noindent \textbf{Lipschitz constant.}
The Lipschitz constant of a covariance diffusion wavelet $h_j(\lambda)$ is the upper bound of its absolute derivative, i.e., 
\begin{equation}
    P \geq |h'_j(\lambda)| = | 2^{j-1}\lambda^{2^{j-1}-1} - 2^{j}\lambda^{2^j-1} |
\end{equation}
for $\lambda \in [0,1]$. Since the powers of $\lambda$ in its domain are always smaller than 1, the largest value of the derivative is obtained for $\lambda=1$ and corresponds to $P = 2^j - 2^{j-1} = 2^{j-1}$.

\noindent \textbf{Maximum.}
We derive the value of $\lambda^*$ corresponding to the maximum of the wavelet with largest scale.
We consider the function:
\[
h_j(\lambda) = \lambda^{2^{j-1}} - \lambda^{2^j}
\]
defined on the interval \( \lambda \in [0,1] \).
Let us define: $a = 2^{j-1}, \quad b = 2^j = 2a$.
Then:
\[
h_j(\lambda) = \lambda^a - \lambda^{2a}.
\]
To find the maximum, since the diffusion wavelet is concave in its domain (cf. Figure~\ref{fig:motivation_example}), we take the derivative:
\[
h_j'(\lambda) = a \lambda^{a-1} - 2a \lambda^{2a - 1} = a \lambda^{a-1} (1 - 2\lambda^a)
\]
and set it equal to zero:
\[
h_j'(\lambda) = 0 \quad \Longleftrightarrow \quad 1 - 2\lambda^a = 0 \quad \Longleftrightarrow \quad x^\lambda = \frac{1}{2}.
\]
Solving for \( \lambda \):
\[
\lambda^* = \left( \frac{1}{2} \right)^{1/a} = \left( \frac{1}{2} \right)^{1/2^{j-1}}.
\]
This justifies the value that we pick for $\gamma $.

\subsection{Monic Cubic Polynomial Wavelets}
Monic cubic polynomial covariance wavelets provide an alternative spectral filtering mechanism that extends the family of covariance wavelets introduced above. Following the construction in \cite{hammond2011wavelets}, these wavelets are defined via a kernel $ h: \mathbb{R} \to \mathbb{R} $ parameterized by $\alpha, \beta, \bar{\lambda}_1,$ and $\bar{\lambda}_2$:

$$
h(\lambda; \alpha, \beta, \bar{\lambda}_1, \bar{\lambda}_2) =
\begin{cases}
\bar{\lambda}_1^{-\alpha} \lambda^{\alpha} & \text{if } \lambda < \bar{\lambda}_1, \\
s(\lambda) & \text{if } \bar{\lambda}_1 \leq \lambda \leq \bar{\lambda}_2, \\
\bar{\lambda}_2^{\beta} \lambda^{-\beta} & \text{if } \lambda > \bar{\lambda}_2,
\end{cases}
$$

\noindent where $h$ is normalized such that $h(\bar{\lambda}_1) = h(\bar{\lambda}_2) = 1$ and $s(\lambda)$ is a cubic polynomial that ensures smooth transitions and is uniquely specified by the conditions:

$$
s(\bar{\lambda}_1) = s(\bar{\lambda}_2) = 1, \quad s'(\bar{\lambda}_1) = \frac{\alpha}{\bar{\lambda}_1}, \quad s'(\bar{\lambda}_2) = -\frac{\beta}{\bar{\lambda}_2}.
$$

The parameters $\bar{\lambda}_1$ and $\bar{\lambda}_2$ are determined from the covariance eigenvalues by selecting the first and third quartiles of the spectrum:

$$
\bar{\lambda}_1 = \lambda_{\lfloor N/4 \rfloor}, \quad \bar{\lambda}_2 = \lambda_{\lceil 3N/4 \rceil},
$$
where $N$ is the number of eigenvalues.

To construct wavelets at multiple scales, a scale function \( t: \{1, \ldots, J\} \to \mathbb{R}_+ \) is used, where the scales \( \{t_j\} \) are logarithmically spaced between
\[
t_1 = \frac{\bar{\lambda}_2 K}{\lambda_1} \quad \text{and} \quad t_J = \frac{\bar{\lambda}_2}{\lambda_1},
\]
with \( K \) controlling the resolution and \( \lambda_1 = \gamma  \) by normalization.

Specifically, the scales are given by:
\[
t_j = t_J \left( \frac{t_1}{t_J} \right)^{\frac{J - j}{J - 1}} = \left( \frac{\bar{\lambda}_2}{\lambda_1} \right) K^{\frac{J - j}{J - 1}}, \quad j = 1, \ldots, J.
\]

Finally, the wavelet coefficients at scale $j$ are defined by applying the kernel $h$ to the scaled eigenvalues $t_j\lambda$, i.e., $h_j(\lambda) = h(t_j\lambda)$. The corresponding wavelet is given by

$$
\mtH_j = \mtV \, h(t_j\mathbf{\Lambda}) \, \mtV^\Tr
$$
where $h(t_j\mathbf{\Lambda})$ acts pointwise on the diagonal entries of $t_j\mathbf{\Lambda}$.
This construction enables flexible spectral filtering, with the parameters $\alpha, \beta,$ and $K$ controlling the shape and resolution of the wavelets across scales.
For monic wavelets, we set the coefficient $\gamma=1$ in the definition of $\mtT$.

\noindent \textbf{Frame bounds.}
From Lemma~\ref{lemma:cov_bounds} we have that the frame bounds depend on the function $G(\lambda)$. Since deriving an analytical expression for $G(\lambda)$ is cumbersome, we plot it in Figure~\ref{fig:g_lambda} (right) and analyze it qualitatively. 
For small $J$ (e.g., $J=3$), there is no wavelet covering the middle eigenvalues and therefore $A^2$ is small (around 0.01 in the plot), whereas $B^2$ is around 1.5. For larger $J$, both $A$ and $B$ increase.

\noindent \textbf{Lipschitz constant.}
The Lipschitz constant \( P \) can be approximated using an upper bound on the supremum of the absolute value of the kernel derivative. For a scaling value \( t_j \) corresponding to the \( j \)-th scale, we have:
\begin{align}
    P = \sup_{\lambda \in [0,1]} \left|h'_j(\lambda)\right| =  \sup_{\lambda \in [0,1]}\left|h'(t_j\lambda)\right|
\end{align}
\begin{align}
h'(t_j\lambda) =
\begin{cases}
\alpha t_j^{\alpha} \bar{\lambda}_1^{-\alpha}\lambda^{\alpha-1} & \text{if } \lambda < \bar{\lambda}_1, \\
s'(t_j\lambda) & \text{if } \bar{\lambda}_1 \leq \lambda \leq \bar{\lambda}_2, \\
-\beta t_j^{-\beta}\bar{\lambda}_2^{\beta} \lambda^{-\beta-1} & \text{if } \lambda > \bar{\lambda}_2,
\end{cases}
\end{align}
Since this kernel has a piecewise definition, the Lipschitz constant is importantly finite and, specifically, it is bounded by the maximum of the absolute value of the three parts:
\begin{align}
    P = \max \left\{\left|\frac{\alpha t_j^\alpha }{ \bar{\lambda}_1}\right|,\ |3at_j^3\bar{\lambda}_2| + |2bt_j\bar{\lambda}_2| + |c|,\ \left|\frac{\beta}{t_j^\beta \bar{\lambda}_2}\right|\right\}
\end{align}
where \( a, b \), and \( c \) are the coefficients of the monic cubic polynomial.

\subsection{Tight Hann Wavelets}

Tight Hann wavelets are constructed by adapting the classical Hann window function to the spectral domain of the covariance matrix. Given a covariance matrix $\mtC$, its wavelet operator $\mtT$ with eigenvalues bounded by $\lambda_1$, and a desired number of wavelet scales $J$, the wavelets are defined using a scaling parameter $R$ that controls the overlap between scales \cite{shuman2015spectrum}.
The spectral translations are given by

$$
t_j = \frac{j \, \lambda_1}{J + 1 - R}, \quad j = 1, \ldots, J,
$$
where each $t_j$ determines the center of the $j$-th wavelet kernel in the spectral domain. The Hann kernel at scale $j$ applied to eigenvalue $\lambda$ is defined as

$$
h_j(\lambda) = \frac{1}{2} + \frac{1}{2} \cos\left( 2\pi \frac{J + 1 - R }{R \lambda_1} (\lambda-t_j) + \pi \right),
$$
with the values truncated to zero outside the kernel support to ensure compact spectral localization.
If spectral warping is applied, commonly by taking the logarithm of eigenvalues, the kernels concentrate resolution in the low-frequency region of the spectrum, enhancing spectral adaptivity.

The wavelet filter matrices at each scale are constructed by multiplying the diagonal kernel matrices with the eigenvector matrix of the covariance matrix, i.e.,

$$
\mtH_j = \mtV \, h_j(\mathbf{\Lambda}) \, \mtV^\top.
$$

This framework yields wavelets with smooth spectral localization and tight frame properties, offering a principled method to define covariance wavelets complementary to polynomial and diffusion wavelet constructions.
For Hann wavelets, we set the parameter $\gamma=10$, which we observe to perform well in practice.

\noindent \textbf{Frame bounds.}
Hann wavelets are tight by design, i.e., they define a frame where $A^2=B^2 = 3R/8$~\cite[Corollary 1]{shuman2015spectrum}. This corresponds to a constant $G(\lambda)$ in Figure~\ref{fig:g_lambda} (where $R=3$).

\noindent \textbf{Lipschitz constant.}
For a scale j, the tight Hann wavelet is differentiable for \( \lambda \in [0,1] \) with its  derivative  given by:
\[
\frac{d h_j}{d\lambda} = -\pi \cdot \frac{J + 1 - R }{R \lambda_1} \cdot \sin\left(2\pi \cdot \frac{J + 1 - R }{R \lambda_1} (\lambda-t_j) + \frac{\pi}{2} \right).
\]
Since the sine function is bounded between 0 and 1, the Lipschitz constant $P$ is:
\begin{align}
    P = \sup_{\lambda \in [0,1]} \left| \frac{d h_j}{d\lambda} \right| = \pi \cdot \left| \frac{J + 1 - R }{R \lambda_{\max}} \right|.
\end{align}

\subsection{Spatial Localization of Covariance Wavelets}

A common property of wavelets is their localization in both space and frequency domains. This is evident in the time domain~\cite{mallat1999wavelet} and has also been extended on grids~\cite{bruna2013invariant} and graphs~\cite{hammond2011wavelets}, where locality is defined by the connectivity structure of the graph.
We show that covariance wavelets enjoy analogous properties. To this end, we first define a notion of locality for covariance matrices based on the distance between two data features $i$ and $j$. Inspired by graph distances, we define the distance between features $i$ and $j$ given a covariance operator $\mtT$ as follows.
\begin{definition}
Consider a covariance wavelet operator $\mtT$ and two nodes on the covariance graph $i,j$. The distance between nodes $i,j$ is defined as
\begin{equation}
    d_\mtT^s(i,j) = |[\mtT^s]_{ij}|^{-1},
\end{equation}    
where $s$ is a positive natural number.
\end{definition}
This distance can be interpreted by looking at the matrix $\mtT$ as defining a graph structure that connects the features of a sample $\vcx$ (cf. Figure~\ref{fig:cov_graph}). In that case, the $ij$-th entry of its powers $[\mtT^s]_{ij}$ contains the sum of the edge weights (i.e., covariance values) along the paths of length $s$ that connect nodes $i$ and $j$. Therefore, features with stronger covariances among each other have smaller distances.  

Let us now consider a wavelet centered at node $a$, i.e., $\vch_j^a = \mtH_j \vcdelta_a$ where $\vcdelta_a$ is a vector with all zeros except for the $a$-th element which is 1. We aim to show that $\vch_j^a$ becomes smaller in magnitude on nodes $b$ which are further away from $a$.

We start by considering diffusion wavelets. The $b$-th value of the diffusion wavelet centered at node $a$ is:
\begin{align}
    |\vcdelta_b^\Tr \vch_j^a | &= |\vcdelta_b^\Tr \mtH_j \vcdelta_a | = |\vcdelta_b^\Tr (\mtT^{2^{j-1}} - \mtT^{2^{j}}) \vcdelta_a| \\
    &= | [\mtT^{2^{j-1}}]_{a,b} - [\mtT^{2^{j}}]_{a,b} | \\
    & \leq d_\mtT^{2^{j-1}}(a,b)^{-1} + d_\mtT^{2^{j}}(a,b)^{-1}.
\end{align}
Therefore, as the distance between $a$ and $b$ increases, the corresponding wavelet coefficient decreases, which shows that diffusion wavelets are localized in the covariance space. 

We consider now a generic wavelet $h_j(\lambda)$ with a zero of integer multiplicity at $\lambda=0$ instantiated at different scales as $h_j(\lambda) = h(t_j\lambda)$ for $t_j\in \mathbb{R}$ (e.g., monic cubic wavelets). We follow a similar approach to~\cite[Theorem 5.7]{hammond2011wavelets} to prove that a wavelet centered at node $a$ vanishes on vertices $b$ sufficiently far from $a$ in the limit of small scales, i.e., $t_j \rightarrow 0$.
Our proof is based on the approximation of $h_j(\lambda)$ with a polynomial for $\lambda \approx 0$. We begin by analyzing the stability of wavelets to perturbations of the wavelet function (which occurs, e.g., in case of approximations).
\begin{lemma}\label{lemma:wavelet_pert_func}
    Let $\vch_j^a$ and $\vchh_j^a$ be the wavelets centered in feature $a$ generated by the functions $h(\lambda)$ and $\schh(\lambda)$, respectively. If $|h(t_j\lambda) - \schh(t_j\lambda)| \leq M(t)$ for all $\lambda \in [0,\gamma]$, then $| [\vch_j^a]_b - [\vchh_j^a]_b | \leq M(t)$ for every feature $b$.
\end{lemma}
\begin{proof}
    From the definition of wavelet localized on a feature, we have that 
    \begin{align}
        | [\vch_j^a]_b - [\vchh_j^a]_b | = 
        |\vcdelta_b^\Tr (\mtH_j - \mthH_j) \vcdelta_a|  \\ 
        = |\vcdelta_b^\Tr (\mtV ( \operatorname{diag}(h(t_j\lambda_1), \dots, h(t_j\lambda_N)) - \nonumber \\
        \operatorname{diag}(\schh(t_j\lambda_1), \dots, \schh(t_j\lambda_N)) \mtV^\Tr) \vcdelta_a| \\
        \leq \| \vcdelta_b \| \| \mtV \| \| \mtV \|  \| \vcdelta_a \| \nonumber \\
        \| \operatorname{diag}(h(t_j\lambda_1) - \schh(t_j\lambda_1), \dots, h(t_j\lambda_N)) - \schh(t_j\lambda_N)) \|
    \end{align}
    where the last step holds from the submultiplicativity of the norm. Since $\| \vcdelta_b \| = \| \mtV \| = \| \vcdelta_a \|=1$, the bound is determined by the norm of the diagonal matrix, i.e., its maximum diagonal element which is bounded by $M(t)$ from the assumptions.
\end{proof}

We now bound the error of approximating $h(\lambda)$ with the first non-zero element of its Taylor polynomial $\schh(\lambda)$.
Let $h$ be differentiable $Z+1$ times and let $h^{(Z)}$ be the non-zero derivative of smallest order of $h(\lambda)$ at $\lambda=0$. We then have
\begin{equation}\label{eq:taylor_approx}
    \schh(t_j\lambda) = \frac{h^{(Z)}(0)}{Z!}(t_j\lambda)^Z.
\end{equation}
From~\cite[Lemma 5.6]{hammond2011wavelets}, we have that
\begin{equation}\label{eq:pert_func_bound}
    |h(t_j\lambda) - \schh(t_j\lambda)| \leq t_j^{Z+1} \frac{\lambda^{Z+1}_1}{(Z+1)!}B
\end{equation}
where $B \geq |h^{(Z+1)}(\lambda)|$ for $\lambda \in [0,t'\gamma]$ and $t'>0$.

With this in place, we provide our localization result. Since all wavelet coefficients decay to $0$ for $t_j=0$ as functions $h(\lambda)$ are band-pass filters, to provide a meaningful bound we consider the quantity $[\vch_j^a]_b/\| \vch_j^a \|$, i.e., the wavelet coefficient normalized by the norm of the wavelet vector. 

\begin{theorem}\label{th:localization}
    Consider a wavelet function $h(\lambda)$ that is differentiable $Z+1$ times with $h(0)=0$, $h^{(z)}(0)=0$ for $z < Z$ and $h^{(Z)}(0)>0$. Assume a constant $t' > 0$ such that $|h^{(Z+1)}(\lambda)| \leq B$ for all $\lambda \in [0,t'\gamma]$.
    Let $\vch_j^a$ be the wavelet produced by function $h(\lambda)$ at scale $t_j$ with covariance operator $\mtT$, centered on feature $a$.
    For a different feature $b$, there exist constants $D, C, t''$ such that
    \begin{equation}
        \frac{[\vch_j^a]_b}{\| \vch_j^a \|} \leq Dt_j + C d_\mtT^Z(a,b)^{-1}
    \end{equation}
    for $t_j \leq \min(t', t'')$.
\end{theorem}
\begin{proof}
    We begin by approximating $h(\lambda)$ with $\schh(\lambda)$ as in~\eqref{eq:taylor_approx}.
    The perturbed wavelet coefficient produced by $\schh(\lambda)$ is
    \begin{align}\label{eq:pert_coeff}
        [\vchh_j^a]_b = \frac{h^{(Z)}(0)}{Z!} t_j^Z \vcdelta_b^\Tr \mtT^Z \vcdelta_a = \frac{h^{(Z)}(0)t_j^Z}{(Z!) d_\mtT^Z(a,b)}.
    \end{align}
    We can then bound the wavelet coefficient as
    \begin{align}
        |[\vch_j^a]_b| &= |[\vch_j^a]_b - [\vchh_j^a]_b + [\vchh_j^a]_b| \\
        & \leq |[\vch_j^a]_b - [\vchh_j^a]_b | + | [\vchh_j^a]_b| \label{eq:second_line} \\
        & \leq  t_j^{Z+1} \frac{\lambda^{Z+1}_1}{(Z+1)!}B  + \frac{h^{(Z)}(0)t_j^Z}{(Z!) d_\mtT^Z(a,b)} \label{eq:numer_bound}
    \end{align}
    where the last step follows by substituting~\eqref{eq:pert_coeff} and~\eqref{eq:pert_func_bound} into~\eqref{eq:second_line}.
    We now consider the norm $\| \vch_j^a \|$. By inverse triangle inequality we can write:
    \begin{align}
        \| \vch_j^a \| &= \| \vchh_j^a + (\vch_j^a - \vchh_j^a) \| \\
        & \geq \| \vchh_j^a \| - \| \vch_j^a - \vchh_j^a \|.
    \end{align}
    We can compute 
    \begin{align}
        \| \vchh_j^a \| = \frac{h^{(Z)}(0)}{Z!} t_j^Z \| \mtT^Z \vcdelta_a \|
    \end{align}
    and, from~~\eqref{eq:pert_func_bound} and Lemma~\ref{lemma:wavelet_pert_func},
    \begin{align}
        \| \vch_j^a - \vchh_j^a \| & \leq \sqrt{N} \max_i | [\vch_j^a - \vchh_j^a]_i | \\
        & \leq \sqrt{N}t_j^{Z+1} \frac{\lambda^{Z+1}_1}{(Z+1)!}B. 
    \end{align}
    This leads to 
    \begin{align}\label{eq:denom_bound}
        \| \vch_j^a \| \geq t_j^Z \left( \frac{h^{(Z)}(0)}{Z!} \| \mtT^Z \vcdelta_a \| - \sqrt{N}t_j \frac{\lambda^{Z+1}_1}{(Z+1)!}B \right).
    \end{align}
    If the r.h.s. of~\eqref{eq:denom_bound} is positive, then we can invert both sides and get an upper bound for $[\vch_j^a]_b/\| \vch_j^a \|$.
    To simplify the notation, we define $\Theta = \sqrt{N}\frac{\lambda^{Z+1}_1}{(Z+1)!}B$ and $\Xi = \frac{h^{(Z)}(0)}{Z!} \| \mtT^Z \vcdelta_a \|$ from~\eqref{eq:denom_bound}, $\Gamma = \frac{\lambda^{Z+1}_1}{(Z+1)!}B$ and $\Psi = \frac{h^{(Z)}(0)}{(Z!)}$ from~\eqref{eq:numer_bound}. For $t_j < \Xi/(2\Theta) = t''$, we get
    \begin{align}
        \frac{[\vch_j^a]_b}{\| \vch_j^a \|} & \leq \frac{t_j \Gamma + \Psi d_\mtT^Z(a,b)^{-1}}{\Xi - t_j\Theta} \\
        & \leq \frac{t_j \Gamma + \Psi d_\mtT^Z(a,b)^{-1}}{\Xi - \Xi/2} = \frac{2(t_j \Gamma + \Psi d_\mtT^Z(a,b)^{-1})}{\Xi}.
    \end{align}
    Choosing $D = 2\Gamma/\Xi$ and $C = 2\Psi/\Xi$ completes the proof.
\end{proof}

Theorem~\ref{th:localization} shows that a covariance wavelet localized on a node $a$ decays on a node $b$ for small scales (i.e., $t_j \rightarrow 0$) according to two terms: one term decreases with rate $\mathcal{O}(t_j)$, whereas the other one decreases with the inverse of the distance between the two nodes $d_\mtT^Z(a,b)^{-1}$. This confirms the localization of covariance wavelets in the covariance space.

\begin{algorithm}[t]
\caption{CST Computation}
\label{alg:cst}
\begin{algorithmic}[1]
\Require Input signal $\vcx_0 \in \mathbb{R}^{N}$, aggregation function $U$, wavelet functions $\{\mtH_j(\cdot)\}_{j=0}^J$, wavelet operator matrix $\mtT$, threshold $\tau$, nonlinearity $\rho$
\Ensure Scattering features $\Phi$

\State Initialize cst\_tree $\gets [\vcx_0]$
\State $\Phi \gets [\,U\vcx_0\,]$

\For{$l = 1$ to $L-1$}
    \State cst\_tree\_next $\gets [\;]$
    \For{$\vcx$ in cst\_tree}
        % \State $\text{input\_energy} \gets \|\text{thisX}\|_F$
        \For{$j = 1$ to $J$}
       \State $\vcx_{(j)} \gets \rho \mtH_j(\mtT)\vcx$
        \If{$\|\vcx_{(j)}\|_F / \| \vcx \|_F > \tau$}

            \State cst\_tree\_next.append($\vcx_{(j)}$)
            \State $\Phi$.append($U\vcx_{(j)}$)
        \EndIf

        \EndFor
    \EndFor
    \State cst\_tree $\gets$ cst\_tree\_next
\EndFor
\State \Return $\Phi$

\end{algorithmic}
\end{algorithm}

\section{Details on CST Computation}

We provide details on the implementation of CSTs in Algorithm~\ref{alg:cst}. Specifically, given a signal $\vcx_0$, the wavelet operator matrix $\mtT$ and the CST parameters, we recursively apply the $j$-th wavelet transform on the representations at the previous layer, and we prune the branches whose relative energy is below the specified threshold $\tau$.

\section{Stability of CSTs to Signal Perturbations}

We extend the stability analysis of CSTs in the main body by considering their response to perturbations of the observed signals $\vcbx = \vcx + \vcdelta$ for a fixed covariance. This scenario occurs, for example, when applying a pre-computed CST on a large unlabeled dataset to a smaller test set, which might contain perturbations or outliers. 
We show the stability of CST to signal perturbations in the following theorem, which is analogous to~\cite{ioannidis2020pruned} for graph scattering transforms.

\begin{theorem}\label{th:stability_sig_pert}
    Consider a CST $\mtPhi(\cdot)$ with $L$ layers and $J$ scales and let $B$ be the largest frame bound among its wavelets. The norm of the CST output difference when computed on a clean and perturbed data sample $\vcx$ and $\vcbx = \vcx + \vcdelta$, respectively, is
    \begin{align}
    \|\mtPhi(\mtT,\vcx) - \mtPhi(\mtT,\vcbx)\| \leq B_U\|\vcdelta\|  \sqrt{ \sum_{\ell=0}^{L-1} F_\ell B^{2\ell} }. \nonumber
    \end{align}
where $F_\ell $ is the number of active scattering features at layer $\ell$ and $\|U\| \leq B_U$.
\end{theorem}

Theorem~\ref{th:stability_sig_pert} establishes that CSTs are affected by signal perturbations proportionally to the size of the perturbation.
Compared to Theorem~\ref{th:cst_stab}, the bound in Theorem~\ref{th:stability_sig_pert} is tighter as it does not depend on the number of layers via $\ell^2$. This is due to the fact that the covariance perturbation affects the wavelets that are applied at each layer, causing the error to be amplified at deeper representations, whereas signal perturbations affect only the representation once.

\section{Proofs}

We report here the proofs and derivations of the theorems, lemmas and propositions in the main paper.

\subsection{Proof of Theorem~\ref{th:perm_equivariance}}

To prove the permutation equivariance of the CST, we begin by proving the permutation equivariance of a single scattering vector $\phi_{j_\ell j_{\ell-1}\dots j_1}(\vcx) = U [\rho (\mtH_{j_\ell} \rho (\mtH_{j_{\ell-1}}\dots \rho (\mtH_{j_1}\vcx)) ) ]$ where $U$ is a permutation equivariant function that preserves the dimension of the input $\vcx$ (e.g., identity matrix).
Consider a dataset $\mtX$ and its permuted version by the permutation matrix $\mtPi$, $\mtbX = \mtPi\mtX$, both zero-mean w.l.o.g..
The sample covariance of the second dataset is 
\begin{align}
    \mtbC = \mtbX \mtbX^\Tr / T &= (\mtPi\mtX)(\mtPi\mtX)^\Tr / T \\
    &= \mtPi (\mtX \mtX^\Tr/T) \mtPi ^\Tr \\
    &= \mtPi \mthC \mtPi ^\Tr,
\end{align}
i.e., it is a permuted version of the sample covariance of the original dataset $\mthC$.
Let us now consider the two wavelet operators $\mthC_N$ and $\mthC_I$ and their permuted versions $\mtbC_N$ and $\mtbC_I$. For $\mthC_N =  \gamma  \mthC / \schw_1$, we have that 
\begin{align}
    \mtbC_N &= \gamma  \mtbC / \schw_1 = \gamma  \mtPi \mthC \mtPi ^\Tr / \schw_1 \\
    &= \mtPi \gamma   \mthC  / \schw_1 \mtPi ^\Tr = \mtPi \mthC_N \mtPi ^\Tr.
\end{align}
Similarly, for $\mthC_I =  \gamma (\mtI-\mthC / w_1)$ we get
\begin{align}
    \mtbC_I &=  \gamma (\mtI-\mtbC / w_1) = \gamma (\mtI-\mtPi \mthC \mtPi ^\Tr / w_1) \\
    &= \mtPi \gamma (\mtI- \mthC  / w_1)\mtPi ^\Tr = \mtPi \mthC_I\mtPi ^\Tr
\end{align}
That is, both $\mtbC_N$ and $\mtbC_I$ are permuted versions of $\mthC_N$ and $\mthC_I$.
We proceed to prove permutation equivariance for a generic wavelet operator $\mthT$ that is either $\mthC_N$ or $\mthC_I$ (we denote its permuted version as $\mtbT$).

Each wavelet $\mtH_j$ can be written as a covariance filter $\mtH_j(\mthT) = \sum_{k=0}^K h_{jk}\mthT^k$ for an appropriate choice of coefficients $h_{jk}$ that depend on the wavelet design (cf. Property~\ref{lemma:cov_wl}).
Once the coefficients $h_{jk}$ are fixed, computing the wavelet on a permuted covariance results in
\begin{align}
    \mtH_j(\mtbT) & =  \sum_{k=0}^K h_{jk}\mtbT^k = \sum_{k=0}^K h_{jk}(\mtPi \mthT \mtPi ^\Tr)^k \\ & = \sum_{k=0}^K h_{jk}\mtPi \mthT^k \mtPi ^\Tr =  \mtPi \sum_{k=0}^K h_{jk} \mthT^k \mtPi ^\Tr \\
    & = \mtPi \mtH_j(\mthT) \mtPi ^\Tr 
\end{align}
where we used the fact that $\mtPi\mtPi^\Tr = \mtI$ since $\mtPi$ is a permutation matrix.
When this wavelet is computed on a permuted input $\mtPi\vcx$, we get
\begin{equation}
    \mtH_j(\mtbT) \mtPi\vcx = \mtPi \mtH_j(\mthT) \mtPi ^\Tr \mtPi\vcx = \mtPi \mtH_j(\mthT) \vcx,
\end{equation}
i.e., the wavelet output is permutation equivariant. 
Since the scattering vector $\phi_{j_\ell j_{\ell-1}\dots j_1}(\vcx)$ is a sequence of wavelets, elementwise nonlinearities $\rho$ (which do not affect feature ordering) and a permutation equivariant function $U$, we conclude that the scattering vector is permutation equivariant. 
By using this fact together with the CST permutation in Definition~\ref{def:cst_perm}, we conclude the proof for permutation equivariance. 
Furthermore, if $U$ is a permutation invariance function (e.g., average or sum), the CST becomes a composition of permutation equivariant and invariant function, resulting in a permutation invariant mapping. 
\qed

\subsection{Proof of Theorem~\ref{lemma:wavelet_stability}}

The proof of Theorem~\ref{lemma:wavelet_stability} follows closely from those of~\cite[Theorem 2]{sihag2022covariance} and~\cite[Theorem 1]{cavallo2024stvnn}, with some changes. We report it fully here for ease of exposition.

Consider a true covariance matrix $\mtC$ and its sample estimation $\mthC$ computed from $T$ observations, from which the wavelet operator $\mthT$ is computed. 
Let us evaluate the effect of the covariance estimation error on both choices of $\mtT$.
For $\mtT = \mtC_N$, we have that $\mthT - \mtT = \gamma  \mthC/\schw_1 - \gamma  \mtC/\scw_1 = \mtE_T$, whereas for $\mtT = \mtC_I$ we have $\mthT - \mtT = \gamma (\mtI - \mthC/\schw_1) - \gamma (\mtI - \mtC/\scw_1) = -\mtE_T$, i.e., the perturbation has the same magnitude in both cases and opposite sign. Since we are eventually interested in upper-bounding the error norm, we provide the following proof for a generic $\| \mtE_T \|$ which holds for both cases, regardless of the sign. 

The wavelet $\mtH_j(\mtT)$ can be instantiated by a polynomial covariance filter of order at most $K=N-1$ (cf. Property~\ref{lemma:cov_wl}).
This filter operates on powers of the operator $\mtT$, which we can analyze via Taylor expansion:
\begin{equation}
    \mthT^\sck = (\mtT + \mtE_T) ^\sck = \mtT^\sck + \sum_{r=0}^{k-1}\mtT^r\mtE_T\mtT^{k-r-1} + \mttE_T
\end{equation}
where $\mttE_T$ is such that $\|\mttE_T\| = \mathcal{O}(\|\mtE_T\|^2)$. For $T$ large enough, we have that $\|\mtE_T\| \ll 1$ and we can ignore $\mttE_T$ in the following.

Let $\mtT = \mtV\mtLambda\mtV^\Tr$ and $\mthT = \mthV\mathbf{\hat{\Lambda}}\mthV^\Tr$ be the eigendecompositions of the true and sample covariance wavelet operator matrix, respectively. 
Let $\mtH(\cdot)$ be a generic covariance filter that implements the wavelet at the scale of interest.
Substituting into the wavelet error the eigendecomposition of $\mtT$ and applying the respective graph Fourier transform on a signal $\vcx$, we get

\begin{align}
\mtH(\mthT)\vcx - \mtH(\mtT)\vcx = \\
\sum_{k=0}^K h_{k} \sum_{r=0}^{k-1}\mtT^r\mtE_T\mtT^{k-r-1}\vcx = \\
     \sum_{i=0}^{N-1}\sctx_{i} \sum_{k=0}^K h_{k} \sum_{r=0}^{k-1}\mtT^r\mtE_T\mtT^{k-r-1}\vcv_i = \\
     \sum_{i=0}^{N-1}\sctx_{i} \sum_{k=0}^K h_{k} \sum_{r=0}^{k-1}\mtT^r\sclambda_i^{k-r-1}\mtE_T\vcv_i \label{eq:evi}
\end{align}
where $\sctx_{i}$ is the $i$-th entry of $\vctx = \mtV^{\Tr}\vcx$. We now expand the last term as
\begin{equation}\label{eq_split}
    \mtE_T\vcv_i = \mtB_i\delta\vcv_i + \delta\lambda_i\vcv_i + (\delta\lambda_i\mtI_N-\mtE_T)\delta\vcv_i
\end{equation}
where
\begin{equation*}
    \mtB_i = \lambda_i\mtI_N-\mtT, \quad \delta\vcv_i = \vchv_i - \vcv_i, \quad \delta\lambda_i = \hat\lambda_i - \lambda_i.
\end{equation*}
Then, substituting \eqref{eq_split} into \eqref{eq:evi}, we get
\begin{align}
    \sum_{i=0}^{N-1}\sctx_{i} \sum_{k=0}^K h_{k} \sum_{r=0}^{k-1}\mtT^r\sclambda_i^{k-r-1}\mtB_i\delta\vcv_i  \label{eq:t1} \\
    +\sum_{i=0}^{N-1}\sctx_{i} \sum_{k=0}^K h_{k} \sum_{r=0}^{k-1}\mtT^r\sclambda_i^{k-r-1}\delta\lambda_i\vcv_i  \label{eq:t2} \\
    +\sum_{i=0}^{N-1}\sctx_{i} \sum_{k=0}^K h_{k} \sum_{r=0}^{k-1}\mtT^r\sclambda_i^{k-r-1}(\delta\lambda_i\mtI_N-\mtE_T)\delta\vcv_i. \label{eq:t3}
\end{align}
In the remaining part of the proof, we proceed with upper-bounding each term individually.

\subsubsection{First term \eqref{eq:t1}.}
We note that $\mtB_i = \lambda_i\mtI_N-\mtT = \mtV(\lambda_i\mtI_N-\mtLambda)\mtV^\Tr$. Plugging this into \eqref{eq:t1}, we get
\begin{align}\label{eq.t1expand}
\begin{split}
    \sum_{i=0}^{N-1}\sctx_{i} \sum_{k=0}^K h_{k} \sum_{r=0}^{k-1}\mtT^r\sclambda_i^{k-r-1}\mtV(\lambda_i\mtI_N-\mtLambda)\mtV^\Tr\delta\vcv_i  \\
    = \sum_{i=0}^{N-1}\sctx_{i} \sum_{k=0}^K h_{k} \sum_{r=0}^{k-1}\sclambda_i^{k-r-1}\mtV\mtLambda^r(\lambda_i\mtI_N-\mtLambda)\mtV^\Tr\delta\vcv_i \\
    = \sum_{i=0}^{N-1}\sctx_{i} \mtV\mtL_{i}\mtV^\Tr(\vchv_i-\vcv_i)
\end{split}
\end{align}
where $\mtL_{i}$ is a diagonal matrix whose $j$-th diagonal element is 0 if $i=j$ and for if $i\neq j$ it is
\begin{align}
\begin{split}
    \mtL_{i,j} &= \sum_{k=0}^K h_{k} \sum_{r=0}^{k-1}\sclambda_i^{k-r-1}\sclambda_j^r(\lambda_i-\lambda_j)
    \\ &= 
    \sum_{k=0}^K h_{k} \frac{\sclambda_i^k-\sclambda_j^k}{\lambda_i-\lambda_j}(\lambda_i-\lambda_j) \\
    &= \sum_{k=0}^K h_{k}\sclambda_i^k - \sum_{k=0}^K h_{k}\sclambda_j^k = h(\lambda_i) - h(\lambda_j).
\end{split}
\end{align}

Therefore, 
\begin{align}
    [\mtL_{i}\mtV^\Tr(\vchv_i-\vcv_i)]_j=\begin{cases}
			0, & \text{if $i=j$}\\
                (h(\lambda_i) - h(\lambda_j))\vcv_j^\Tr\vchv_i, & \text{if $i \neq j$}
		 \end{cases}
\end{align}
Taking the norm of \eqref{eq.t1expand}, we get 
\begin{align}
    \left\|\sum_{i=0}^{N-1}\sctx_{i} \mtV\mtL_{i}\mtV^\Tr(\vchv_i-\vcv_i)\right\| \le \\ \sum_{i=0}^{N-1}|\sctx_{i}| \|\mtV\|~\|\mtL_{i}\mtV^\Tr(\vchv_i-\vcv_i)\| \leq \\
    \sqrt{N}\sum_{i=0}^{N-1}|\sctx_{i}|\max_j|h(\lambda_i) - h(\lambda_j)||\vcv_j^\Tr\vchv_i| 
\end{align}
where the first inequality derives from the triangle and Cauchy-Schwarz inequalities and the second one holds from $\|\mtV\| = 1$ alongside $\|\vcy\|\leq\sqrt{N}\max_{i=1}y_i$ for an arbitrary vector $\vcy \in \mathbb{R}^N$.
We now leverage the result from \cite[Theorem~4.1]{loukas2017howclose} to characterize the dot product of the eigenvectors of true and sample covariance matrices under As.~\ref{as_eig_diff} as
\begin{align}\label{eq.event1}
    \mathbb{P}(|\vcv_j^\Tr\vchv_i|\geq B) \leq \frac{1}{T}\left(\frac{2k_j}{B|\lambda_i-\lambda_j|}\right)^2
\end{align}
where $k_j=\left( \mathbb{E}[||\vcx\vcx^\Tr\vcv_j||^2_2]-\lambda_j^2 \right)^{1/2}$ is related to the kurtosis of the data distribution~\cite{loukas2017howclose,sihag2022covariance}.
By setting 
\begin{align}
    B = \frac{2k_je^{\epsilon/2}}{T^{1/2}|\lambda_i-\lambda_j|},
\end{align}
we get
\begin{align}
    \max_j|h(\lambda_i) - h(\lambda_j)||\vcv_j^\Tr\vchv_i| \leq \\ \max_j\frac{|h(\lambda_i) - h(\lambda_j)|}{|\lambda_i-\lambda_j|}\frac{2k_je^{\epsilon/2}}{T^{1/2}}
\end{align}
with probability at least $1-e^{-\epsilon}$.

All wavelet functions $h(\lambda)$ are Lipschitz with constant $P$. %, i.e.,
Therefore, the term in \eqref{eq:t1} is bounded as
\begin{align}
\begin{split}
\label{eq:first_term}
    \left\|\sum_{i=0}^{N-1}\sctx_{i} \sum_{k=0}^K h_{k} \sum_{r=0}^{k-1}\mtC^r\sclambda_i^{k-r-1}\beta_i\delta\vcv_i\right\| \leq \\
    \sum_{i=0}^{N-1}|\sctx_{i}| \frac{2P\sqrt{N} k_{\text{max}}e^{\epsilon/2}}{T^{1/2}} \leq \frac{2}{T^{1/2}}Pk_{\text{max}}e^{\epsilon/2}N \|\vcx\|
\end{split}
\end{align}
with probability at least $1-e^{-\epsilon}$. Note that we leveraged $\sum_{i=0}^{N-1} |\tilde{x}_{i}| \leq \sqrt{N}\|\vcx\|$, and defined $k_{\text{max}} := \max_jk_j$.

\medskip
\noindent\textbf{Second term \eqref{eq:t2}.}
We rewrite \eqref{eq:t2} as
\begin{align}
    \sum_{i=0}^{N-1}\sctx_{i} \sum_{k=0}^K h_{k} \sum_{r=0}^{k-1}\mtT^r\sclambda_i^{k-r-1}\delta\lambda_i\vcv_i = \\
    \sum_{i=0}^{N-1}\sctx_{i} \sum_{k=0}^K h_{k} \sum_{r=0}^{k-1}\sclambda_i^r\sclambda_i^{k-r-1}\delta\lambda_i\vcv_i = \\
     \sum_{i=0}^{N-1}\sctx_{i} \sum_{k=0}^K kh_{k}\sclambda_i^{k-1}\delta\lambda_i\vcv_i = \\
    \sum_{i=0}^{N-1}\sctx_{i} h'(\lambda_i)\delta\lambda_i\vcv_i.
\end{align}
where $h'(\lambda)$ is the derivative of $h(\lambda)$ w.r.t. $\lambda$.
Taking the norm and applying standard inequalities, we get 
\begin{equation}
    \left\|\sum_{i=0}^{N-1}\sctx_{i} h'(\lambda_i)\delta\lambda_i\vcv_i\right\| \le \sum_{i=0}^{N-1}|\sctx_{i}|~ |h'(\lambda_i)|~|\delta\lambda_i|~\|\vcv_i\|.
\end{equation}
Here, we have that $\|\vcv_i\| = 1$, and that derivative of the wavelet $h'(\lambda)$ is bounded by $P$ from the Lipschitz property. We now proceed with bounding $|\delta\lambda_i|$. From Weyl's theorem~\cite[Theorem 8.1.6]{golub13}, we note that $\|\mtE_T\|\leq \alpha$ implies that $|\delta\lambda_i|\leq \alpha$ for any $\alpha > 0$ since $\mtT$ is a positive semi-definite matrix.
Then, we observe that 
\begin{align}
\| \mtE_T \| & = \| \gamma ( \mthC/\schw_1 - \mtC/\scw_1 ) \|  \\
& = \gamma  \| \mthC/\schw_1 - \mtC/\scw_1  \|  \\
& = \gamma  \| \mthC/\schw_1 - \mtC/\scw_1 - \mthC/\scw_1 + \mthC/\scw_1  \|  \\
& \leq \gamma  (\| \mthC/\scw_1 - \mtC/\scw_1 \| + \| \mthC/\schw_1 - \mthC/\scw_1  \| )  \\
& = \gamma  \left(\frac{1}{w_1} \| \mthC - \mtC \| + \|\mthC\| \| 1/\schw_1  - 1/\scw_1  \| \right)  \\
& = \gamma  \left(\frac{1}{w_1} \| \mthC - \mtC \| + \schw_1 \| 1/\schw_1  - 1/\scw_1  \| \right)  \\
& = \gamma  \left(\frac{1}{w_1} \| \mthC - \mtC \| + |\scw_1 - \schw_1| / \scw_1 \right),
\end{align}
Where we used the fact that $\|\mthC\| = \schw_1$ by definition of spectral norm.
We can again apply Weyl's theorem to note that if $|\scw_1 - \schw_1| \leq \beta$ then $\|\mthC - \mtC\| = \| \mtE_C\| \leq \beta$ and, therefore, 
\begin{align}
    \| \mtE_T \| \leq \frac{2\gamma }{w_1} \| \mtE_C \|.
\end{align}

Next,  $\| \mtE_C \|$ can be bounded via the result from~\cite[Theorem 5.6.1]{vershynin2018high}: 
\begin{align}\label{eq.event2}
    \footnotesize
    \mathbb{P}\bigg(\|\mtE_C\| \leq \underbrace{Q\left( \sqrt{\frac{G^2N(\log N+u)}{T}}+\frac{G^2N(\log N+u)}{T} \right)\|\mtC\|}_{\alpha} \bigg) \\  \geq 1-2e^{-u}.
\end{align}
where $Q$ is an absolute constant and $G\geq 1$ derives from Assumption~\ref{as_norm}. Finally, $|\tilde{x}_{i}|$ is handled via inequality $\sum_{i=0}^{N-1} |\tilde{x}| \leq \sqrt{N}\|\vcx\|$. Putting all these together, we upper-bound the second term \eqref{eq:t2} by
\begin{align}
\label{eq:second_term}
\begin{split}
    \left\|\sum_{i=0}^{N-1}\sctx_{i} \sum_{k=0}^K h_{k} \sum_{r=0}^{k-1}\mtT^r\sclambda_i^{k-r-1}\delta\lambda_i\vcv_i\right\| \leq \\
    \frac{P\sqrt{N}Q2\gamma }{w_1} \left( 
    \sqrt{\frac{G^2N(\log N+u)}{T}} + 
\right. \\ 
\left. 
    \frac{G^2N(\log N+u)}{T} 
\right) \|\mtC\| \| \vcx \|
\end{split}
\end{align}
with probability at least $1-2e^{-u}$.

\smallskip
\noindent\textbf{Third term \eqref{eq:t3}.} By leveraging equations (65)-(68) in \cite{sihag2022covariance}, we can show that $\|(\delta\lambda_i\mtI_N-\mtE_T)\delta\vcv_i\|$ scales as $\mathcal{O}(1/T)$ and is therefore negligible compared to the other terms scaling as $\mathcal{O}(1/\sqrt{T})$.

\smallskip
Bringing together \eqref{eq:first_term}, \eqref{eq:second_term} and the observation that \eqref{eq:t3} scales as $\mathcal{O}(1/T)$ leads to the bound in Theorem~\ref{lemma:wavelet_stability}, where we do not report the term $\| \vcx \|$ as per the definition of wavelet stability in~\eqref{eq:stab_def}.
The bound holds given that events \eqref{eq.event1} and \eqref{eq.event2} are independent.
Although their independence might not hold in practice, the factors involved in the bound remain the same, which allows for an analysis of how different design choices and data characteristics affect the CST stability. This consideration is analogous to~\cite[Remark 2]{cavallo2024stvnn}. 
\qed

\subsection{Proof of Proposition~\ref{lemma:prune_cond}}

We provide a condition such that the pruned tree of the CST computed on the true and estimated covariance have the same branches.
Let $\vcx_{(j_\ell,\dots,j_1)}$ be the scattering vector at layer $\ell$  for an input sample $\vcx$ to which the wavelets at scales $j_1, \dots, j_\ell$ have been applied subsequently, and let $\vchx_{(j_\ell,\dots,j_1)}$ be the same scattering vector with wavelets computed on the estimated covariance $\mthC$.
We begin by stating and proving the following lemma, which extends~\cite[Lemma 4]{ioannidis2020efficient} to the CST.
\begin{lemma}\label{lemma:vector_pert}
    Consider the scattering features $\vcx_{(j_\ell,\dots,j_1)}$ and $\vchx_{(j_\ell,\dots,j_1)}$ computed from true and perturbed covariance operators, respectively, at the $\ell$-th layer of a CST $\mtPhi(\mtT, \vcx)$. Let $B$ and $\Delta$ be the largest frame and stability bounds of the wavelets in the CST and $P$ the maximum Lipschitz constant. We have that 
    \begin{equation}\label{eq:vector_pert}
    \|\vchx_{(j_\ell,\dots,j_1)} - \vcx_{(j_\ell,\dots,j_1)}\| \leq \ell \Delta B^{\ell-1}\|\vcx\|.
\end{equation}
\end{lemma}
\begin{proof}
    We begin by expanding $\vchx_{(j_\ell,\dots,j_1)}, \vcx_{(j_\ell,\dots,j_1)}$ via their definition, adding and subtracting within the norm on the left-hand side the term $\rho(\mtH_{j_\ell}(\mtT)\rho(\mtH_{j_{\ell-1}}(\mthT) \dots  (\mtH_{j_{1}}(\mthT)\vcx )\dots ))$ and using triangle inequality:
    \begin{align}
        \|\vchx_{(j_\ell,\dots,j_1)} - \vcx_{(j_\ell,\dots,j_1)}\| \leq \\
        \| \rho(\mtH_{j_\ell}(\mthT)\rho(\mtH_{j_{\ell-1}}(\mthT) \dots  (\mtH_{j_{1}}(\mthT)\vcx) \dots ) \\ 
        -\rho(\mtH_{j_\ell}(\mtT)\rho(\mtH_{j_{\ell-1}}(\mtT) \dots  (\mtH_{j_{1}}(\mtT)\vcx)\dots )) \\
        +\rho(\mtH_{j_\ell}(\mtT)\rho(\mtH_{j_{\ell-1}}(\mthT) \dots  (\mtH_{j_{1}}(\mthT)\vcx)\dots )) \\
        -\rho(\mtH_{j_\ell}(\mtT)\rho(\mtH_{j_{\ell-1}}(\mthT) \dots  (\mtH_{j_{1}}(\mthT)\vcx)\dots )) \| \leq \\ 
        \| \rho(\mtH_{j_\ell}(\mtT) ( \rho(\mtH_{j_{\ell-1}}(\mtT) \dots  (\mtH_{j_{1}}(\mtT)\vcx)\dots )) \label{eq:first_sum1}  \\  
        - \rho(\mtH_{j_\ell}(\mtT)\rho(\mtH_{j_{\ell-1}}(\mthT) \dots  (\mtH_{j_{1}}(\mthT)\vcx)\dots ))  \| \label{eq:first_sum2} \\
        + \| \rho(\mtH_{j_\ell}(\mtT) ( \rho(\mtH_{j_{\ell-1}}(\mthT) \dots  (\mtH_{j_{1}}(\mthT)\vcx)\dots ))  \label{eq:second_sum1} \\ 
        - \rho(\mtH_{j_\ell}(\mthT)\rho(\mtH_{j_{\ell-1}}(\mthT) \dots  (\mtH_{j_{1}}(\mthT)\vcx)\dots ))  \|. \label{eq:second_sum2}
    \end{align}

    For the term in~\eqref{eq:first_sum1}-\eqref{eq:first_sum2}, since the nonlinearity is non-expansive (i.e., $\|\rho(\cdot)\|\leq 1$) and the norm is submultiplicative, it holds that 
    \begin{align}
        \| \rho(\mtH_{j_\ell}(\mtT) ( \rho(\mtH_{j_{\ell-1}}(\mtT) \dots  (\mtH_{j_{1}}(\mtT)\vcx)\dots ))  \\  
        - \rho(\mtH_{j_\ell}(\mtT)\rho(\mtH_{j_{\ell-1}}(\mthT) \dots  (\mtH_{j_{1}}(\mthT)\vcx)\dots )) \| \leq \\
        \|\rho(\cdot)\| \| \mtH_{j_\ell}(\mtT)  \rho(\mtH_{j_{\ell-1}}(\mtT) \dots  (\mtH_{j_{1}}(\mtT)\vcx)\dots )  \\  
        - \mtH_{j_\ell}(\mtT)\rho(\mtH_{j_{\ell-1}}(\mthT) \dots  (\mtH_{j_{1}}(\mthT)\vcx)\dots ) \| \leq \\
        \| \mtH_{j_\ell}(\mtT) \| \|  \rho(\mtH_{j_{\ell-1}}(\mtT) \dots  (\mtH_{j_{1}}(\mtT)\vcx)\dots )  \\
        - \rho(\mtH_{j_{\ell-1}}(\mthT) \dots  (\mtH_{j_{1}}(\mthT)\vcx)\dots ) \| = \\
        \| \mtH_{j_\ell}(\mtT) \| \| \vcx_{(j_{\ell-1},\dots,j_1)} - \vchx_{(j_{\ell-1}l,\dots,j_1)} \| \leq \\
        B \| \vcx_{(j_{\ell-1},\dots,j_1)} - \vchx_{(j_{\ell-1}l,\dots,j_1)} \|\label{eq:last_term1}
    \end{align}
    where the last step holds by applying the frame bound.
    For the term in~\eqref{eq:second_sum1}-\eqref{eq:second_sum2}, instead, we get
    \begin{align}
        \| \rho(\mtH_{j_\ell}(\mtT) ( \rho(\mtH_{j_{\ell-1}}(\mthT) \dots  (\mtH_{j_{1}}(\mthT)\vcx)\dots ))  \\ 
        - \rho(\mtH_{j_\ell}(\mthT)\rho(\mtH_{j_{\ell-1}}(\mthT) \dots  (\mtH_{j_{1}}(\mthT)\vcx)\dots ))  \| = \\
        \| \rho(\mtH_{j_\ell}(\mtT) \vchx_{(j_{\ell-1},\dots,j_1)})  
        - \rho(\mtH_{j_\ell}(\mthT)\vchx_{(j_{\ell-1},\dots,j_1)})  \| \leq \\
        \| \mtH_{j_\ell}(\mtT) - \mtH_{j_\ell}(\mthT) \|\| \vchx_{(j_{\ell-1},\dots,j_1)}\| \label{eq:last_term2}
    \end{align}
    where $\| \mtH_{j_\ell}(\mtT) - \mtH_{j_\ell}(\mthT) \| \leq \Delta$ from Theorem~\ref{lemma:wavelet_stability}.
    We can also bound $\| \vchx_{(j_{\ell-1},\dots,j_1)}\|$ as follows:
    \begin{align}
        \| \vchx_{(j_{\ell-1},\dots,j_1)}\| \leq \\
        \| \rho(\mtH_{j_{\ell-1}}(\mthT) \dots  (\mtH_{j_{1}}(\mthT)\vcx)\dots ) \| \leq \\
        \| \rho(\cdot) \| \| \mtH_{j_{\ell-1}}(\mthT) \| \dots \| \mtH_{j_{1}}(\mthT) \| \| \vcx \| \leq \\
        B^{\ell-1}\| \vcx \| \label{eq:last_term3}
    \end{align}
    using again the non-expansiveness of the nonlinearity and the submultiplicativity of the norm.
    
    Plugging~\eqref{eq:last_term1},\eqref{eq:last_term2} and \eqref{eq:last_term3} into \eqref{eq:first_sum1}-\eqref{eq:second_sum2} leads to the recursive inequality
    \begin{align}
        \|\vchx_{(j_\ell,\dots,j_1)} - \vcx_{(j_\ell,\dots,j_1)}\| \leq \\
        B \| \vcx_{(j_{\ell-1},\dots,j_1)} - \vchx_{(j_{\ell-1}l,\dots,j_1)}\| + \Delta B^{\ell-1}\| \vcx \|
    \end{align}
    which can be solved to obtain~\eqref{eq:vector_pert}.
\end{proof}

We now proceed with computing a sufficient condition for the pruning trees to be indentical under covariance perturbations.
Following the derivations in~\cite[(56)-(63)]{ioannidis2020efficient}, we get that a sufficient condition for the true and perturbed scattering tree to be identical is analogous to~\cite[(17)]{ioannidis2020efficient}:
\begin{align}\label{eq:cond1}
    | \| \mtH_j(\mtT)\vcx_{(j_\ell,\dots,j_1)} \|^2 - \tau \| \vcx_{(j_\ell,\dots,j_1)} \|^2 | > \nonumber \\ ((\ell+1)\Delta B^{\ell} \|\vcx\|)^2 + \tau\|\vchx_{(j_\ell,\dots,j_1)} - \vcx_{(j_\ell,\dots,j_1)}\|^2.
\end{align}
From~\eqref{eq:vector_pert}, we get that~\eqref{eq:cond1} is satisfied if
\begin{align}
    |\| \mtH_j(\mtC)\vcx_{(j_\ell,\dots,j_1)} \|^2 - \tau \| \vcx_{(j_\ell,\dots,j_1)} \|^2 | > \nonumber \\ 
    ((\ell+1)\Delta B^{\ell} \|\vcx\|)^2 + \tau (\ell \Delta B^{\ell-1}\|\vcx\|)^2 = \\
    (\Delta B^{\ell-1} \|\vcx\|)^2 ( (\ell+1)B + \ell\tau ). \nonumber
\end{align}
\qed

\subsection{Proof of Theorem~\ref{th:cst_stab}}

We now prove the stability of the CST to covariance estimation errors, i.e., we upper-bound $\| \mtPhi(\mtT,\vcx) - \mtPhi(\mthT,\vcx) \|$.
From the definition of CST and spectral norm, assuming no CST features are pruned, we have
\begin{align}
    \|\mtPhi(\mtT,\vcx) - \mtPhi(\mthT,\vcx)\| = \\
    \sqrt{\| U\vcx - U\vcx \|^2 + \| U\vcx_{(1)} - U\vchx_{(1)}\|^2 + \dots}  \\ 
    \overline{+ 
    \|  U\vcx_{(J,\dots,J)} - U\vchx_{(J,\dots,J)}\|^2} \leq \\
    \|U\|\sqrt{\sum_{\ell,j_\ell,\dots,j_1} \| \vcx_{(j_\ell,\dots,j_1)} - \vchx_{(j_\ell,\dots,j_1)} \|^2 }. \label{eq:norm1}
\end{align}

The first term in the sum is zero, whereas all other terms can be bounded using Lemma~\ref{lemma:vector_pert}.
This bounds each single term in the sum in~\eqref{eq:norm1}. However, some of these terms can be pruned by the CST. Let $F_\ell $ be the number of selected features at the $\ell$-th CST layer after pruning. 
By plugging this in~\eqref{eq:norm1} and defining $\|U\| \leq B_U$, we get the final bound:
\begin{align}
    \|\mtPhi(\mtC,\vcx) - \mtPhi(\mthC,\vcx)\| \leq B_U\Delta\|\vcx\|\sqrt{\sum_{\ell=1}^{L-1} \ell^2 B^{2\ell-2}F_\ell }.
\end{align}
\qed

\subsection{Proof of Theorem~\ref{th:stability_sig_pert}}

We now provide a bound for the CST under signal perturbation, i.e., we bound the quantity $\|\mtPhi(\mtT,\vcx) - \mtPhi(\mtT,\vcbx)\|$ where $\vcbx = \vcx + \vcdelta$.
Let $\vcbx_{(j_\ell,\dots,j_1)}$ be the scattering features computed by applying the wavelets at scales $j_\ell,\dots,j_1$ to the perturbed signal $\vcbx$. By applying the non-expansiveness of the nonlinearity $\rho$ and the fact that the norm is submultiplicative, we have that 
\begin{align}
    \| \vcx_{(j_\ell,\dots,j_1)} - \vcbx_{(j_\ell,\dots,j_1)} \| = \\
    \| \rho(\mtH_{j_\ell}(\mtT) \vcx_{(j_{\ell-1},\dots,j_1)}) -\rho(\mtH_{j_\ell}(\mtT) \vcbx_{(j_{\ell-1},\dots,j_1)}) \| \leq \\
    \| \rho(\cdot) \| \| \mtH_{j_\ell}(\mtT) \| \| \vcx_{(j_{\ell-1},\dots,j_1)} - \vcbx_{(j_{\ell-1},\dots,j_1)} \| \leq \\
    B\| \vcx_{(j_{\ell-1},\dots,j_1)} - \vcbx_{(j_{\ell-1},\dots,j_1)} \|.
\end{align}
By repeating this step $\ell$ times we get 
\begin{equation}\label{eq:sig_pert_bound}
    \| \vcx_{(j_\ell,\dots,j_1)} - \vcbx_{(j_\ell,\dots,j_1)} \| \leq B^\ell\|\vcdelta\|.
\end{equation}
We can expand the stability bound, ignoring for now the pruning, as
\begin{align}
    \|\mtPhi(\mtT,\vcx) - \mtPhi(\mtT,\vcbx)\| = \\
    \sqrt{\| U\vcx - U\vcbx \|^2 + \| U\vcx_{(1)} - U\vcbx_{(1)}\|^2 + \dots}  \\ 
    \overline{+ 
    \|  U\vcx_{(J,\dots,J)} - U\vcbx_{(J,\dots,J)}\|^2} \leq \\
    \|U\|\sqrt{\sum_{\ell,j_\ell,\dots,j_1} \| \vcx_{(j_\ell,\dots,j_1)} - \vcbx_{(j_\ell,\dots,j_1)} \|^2 }. \label{eq:norm2}
\end{align}
Each of the terms in~\eqref{eq:norm2} is bounded as in~\eqref{eq:sig_pert_bound}. Let $F_\ell $ be the number of selected scattering features at layer $\ell$ after pruning, and $\|U\|\leq B_U$. We get
\begin{equation}
    \|\mtPhi(\mtT,\vcx) - \mtPhi(\mtT,\vcbx)\| \leq B_U\|\vcdelta\|  \sqrt{ \sum_{\ell=0}^{L-1} F_\ell B^{2\ell} }.
\end{equation}
\qed

\section{Further Details on the Experiments}

We provide additional details on the task of age prediction from cortical thickness data, the datasets used and details on the experimental setup.

\subsection{Background on Age Prediction from Cortical Thickness}

Cortical thickness is the distance between the outer surface of the brain and the white matter surface and it measures the thickness of the cerebral cortex (typically between 2 and 4 millimeters for healthy adults).
It is usually estimated from high-resolution T1-weighted MRI scans via image processing tools that segment the brain into the different tissue types (gray matter, white matter and cerebrospinal fluid), reconstruct the cortical surfaces, and compute thickness by measuring the distance between the white matter and pial surfaces at thousands of points across the cortex. These measurements are typically averaged within brain regions, resulting in one feature per region. The number of brain regions depends on the atlas used in the cortical thickness computation, and ranges between 62 and 68 for the datasets in this work (we report more details later). 
Cortical thickness is a useful biomarker for studying brain aging because the cortex tends to thin with age, especially in frontal and temporal regions. This makes it valuable for predicting a person’s age using machine learning models that map patterns of cortical thickness to chronological age. Such models can help track healthy aging, detect early signs of neurodegeneration, or identify individuals with accelerated brain aging due to neurological or psychiatric conditions.
In particular, age predictors from cortical thickness measures trained on a healthy population can be informative of neurodegenerative cases and various illnesses via their ability to detect uncommon aging patterns in the brain. Covariance information across cortical thickness measures plays a relevant role in this direction, considering the complex underlying mechanisms that regulate brain behavior. Covariance-based predictors have been shown useful to analyze and process cortical thickness measurements in several works~\cite{yin2023anatomically,bashyam2020mri,couvy2020ensemble}.
In~\cite{sihag2024explainable}, covariance-based neural networks obtain good age prediction performance from cortical thickness measures and identify neurodegenerative conditions effectively via brain aging. This highlights the impact of the task at hand.

\subsection{Dataset Details}

\noindent \textbf{ADNI.} 
ADNI~\cite{jack2008alzheimer} is a project that collects clinical measurements of patients affected by Alzheimer's disease and healthy control patients. It is divided in multiple phases based on the collection protocols and historical periods. Here, we analyze two collections: 
ADNI1 contains 801 patients and cortical thickness features for 68 brain regions, while ADNI2 contains 1142 patients and cortical thickness features for 68 brain regions.  
Cortical thickness measures for both datasets are downloaded from \url{https://ida.loni.usc.edu/}, which provides the outputs of the FreeSurfer~\cite{fischl2012freesurfer} processing of the brain MRI recordings. For each patient, we consider the first scan available, generally corresponding to the screening visit. 

\noindent \textbf{PPMI.} PPMI~\cite{marek2011parkinson} collects clinical data of patients suffering from Parkinson's disease and healthy control patients. 
We download the cortical thickness features from \url{https://ida.loni.usc.edu/}, which provides measures for 1704 patients and over 68 brain regions.   
Data used in the preparation of this article was obtained on 2025-07-01 from the Parkinson’s
Progression Markers Initiative (PPMI) database (\url{www.ppmi-info.org/access-dataspecimens/download-data}), RRID:SCR\_006431. For up-to-date information on the study, visit
\url{www.ppmi-info.org}.
PPMI – a public-private partnership – is funded by the Michael J. Fox Foundation for Parkinson’s Research
and funding partners, including 4D Pharma, Abbvie, AcureX, Allergan, Amathus Therapeutics, Aligning
Science Across Parkinson's, AskBio, Avid Radiopharmaceuticals, BIAL, BioArctic, Biogen, Biohaven,
BioLegend, BlueRock Therapeutics, Bristol-Myers Squibb, Calico Labs, Capsida Biotherapeutics, Celgene,
Cerevel Therapeutics, Coave Therapeutics, DaCapo Brainscience, Denali, Edmond J. Safra Foundation, Eli
Lilly, Gain Therapeutics, GE HealthCare, Genentech, GSK, Golub Capital, Handl Therapeutics, Insitro, Jazz
Pharmaceuticals, Johnson \& Johnson Innovative Medicine, Lundbeck, Merck, Meso Scale Discovery, Mission
Therapeutics, Neurocrine Biosciences, Neuron23, Neuropore, Pfizer, Piramal, Prevail Therapeutics, Roche,
Sanofi, Servier, Sun Pharma Advanced Research Company, Takeda, Teva, UCB, Vanqua Bio, Verily, Voyager.

\noindent \textbf{Abide.} 
Abide~\cite{craddock2013neuro} contains clinical information for patients suffering from autism and healthy control patients. We download the preprocessed cortical thickness measures following the instructions at \url{http://preprocessed-connectomes-project.org/abide/} and we keep the thickness features of the Mindboggle protocol~\cite{klein2017mindboggling}. The final dataset consists of 1035 patients with cortical thickness features for 62 brain regions.  

\subsection{Implementation Details}
Experiments are run on AMD EPYC 7452 32-Core Processor with 10 GB of RAM and on a 13th Gen Intel Core i7-1365U with 16 GB of RAM.
Our code is implemented in Python using the NumPy~\cite{harris2020array} and Scikit-learn~\cite{scikit-learn} libraries. 
All experiments are repeated 10 times for different seeds and we report average values and standard deviations. 
We normalize all data to zero mean and unit variance as a preprocessing step. 
We use Mean Average Error (MAE) to assess regression performance due to its ease of interpretation and popularity in related works~\cite{sihag2022covariance}, and Mean Squared Error (MSE) to evaluate the stability of the representations.

\begin{table}[t]
\centering
% \footnotesize
\begin{tabular}{c|l|cccc}
\toprule
& & \textbf{ADNI1} & \textbf{ADNI2} & \textbf{PPMI} & \textbf{ABIDE} \\
\midrule
\multirow{4}{*}{\textbf{Diff CST}} 
& $J$ & 4 & 6 & 7 & 7 \\
& $L$  & 4 & 2 & 4 & 2 \\
& $\mtT$ & $\mtC_N$ & $\mtC_N$ & $\mtC_I$ & $\mtC_I$ \\
& $\alpha_R$ & 200 & 100 & 100 & 10 \\
\midrule
\multirow{4}{*}{\textbf{Hann CST}} 
& $J$ & 6 & 7 & 7 & 4 \\
& $L$ & 2 & 2 & 4 & 2 \\
& $\mtT$ & $\mtC_I$  & $\mtC_I$ & $\mtC_I$ & $\mtC_I$ \\
& $\alpha_R$ & 200 & 200 & 200 & 100 \\
\midrule
\multirow{4}{*}{\textbf{Monic CST}} 
& $J$ & 7 & 4 & 7 & 7 \\
& $L$ & 2 & 4 & 4 & 3 \\
& $\mtT$ & $\mtC_I$ & $\mtC_I$ & $\mtC_I$ & $\mtC_I$ \\
& $\alpha_R$ & 200 & 100 & 100 & 100 \\
\midrule
\multirow{3}{*}{\textbf{VNN}} 
& $K$ & 2 & 2 & 2 & 2 \\
& $L$ & 1 & 1 & 1 & 1 \\
& $F$ & 32 & 32 & 32 & 32 \\
\midrule
\multirow{3}{*}{\textbf{MLP}} 
& $\alpha_R$ & 200 & 100 & 100 & 100 \\
& $L$ & 1 & 1 & 1 & 2 \\
& $F$ & 32 & 32 & 32 & 32 \\
\midrule
\multirow{2}{*}{\textbf{PCA}} 
& $k$ & 20 & 10 & 10 & 50 \\
& $\alpha_R$ & 200 & 200 & 100 & 100 \\
\midrule
\multirow{1}{*}{\textbf{Raw}} 
& $\alpha_R$ & 200 & 200 & 100 & 100 \\
\bottomrule
\end{tabular}
\caption{Selected parameters on real datasets.}
\label{tab:params}
\end{table}

\begin{table}[t]
\centering
% \footnotesize
\begin{tabular}{c|l|ccc}
\toprule
& & \textbf{Tail 0.1} & \textbf{Tail 0.5} & \textbf{Tail 0.9}\\
\midrule
\multirow{4}{*}{\textbf{Diff CST}} 
& $J$ & 7 & 4 & 7 \\
& $L$  & 2 & 2 & 2 \\
& $\mtT$ & $\mtC_N$ & $\mtC_I$ & $\mtC_N$ \\
& $\alpha_R$ & 1 & 1 & 1 \\
\midrule
\multirow{4}{*}{\textbf{Hann CST}} 
& $J$ & 5 & 7 & 4 \\
& $L$ & 2 & 2 & 2 \\
& $\mtT$ & $\mtC_I$  & $\mtC_I$ & $\mtC_N$ \\
& $\alpha_R$ & 1 & 1 & 1 \\
\midrule
\multirow{4}{*}{\textbf{Monic CST}} 
& $J$ & 4 & 4 & 4  \\
& $L$ & 3 & 2 & 2 \\
& $\mtT$ & $\mtC_N$ & $\mtC_N$ & $\mtC_N$ \\
& $\alpha_R$ & 1 & 1 & 1 \\
\midrule
\multirow{2}{*}{\textbf{PCA}} 
& $k$ & 20 & 20 & 20 \\
& $\alpha_R$ & 1 & 1 & 1 \\
\midrule
\multirow{1}{*}{\textbf{Raw}} 
& $\alpha_R$ & 1 & 1 & 1 \\
\bottomrule
\end{tabular}
\caption{Selected parameters on synthetic datasets.}
\label{tab:params_synth}
\end{table}

\begin{figure}[t]
    \centering
    \includegraphics[width=.7\linewidth]{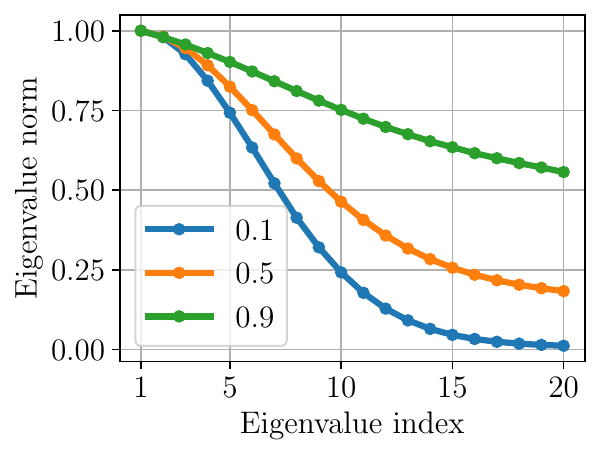}
    \caption{Eigenvalue distribution for different tail strengths.} 
    \label{fig:eig_distr}        
\end{figure}

\subsection{Hyperparameters}

We provide additional details on model implementations and hyperparameters. The best hyperparameters selected via grid search on a validation set for the various models on real datasets are reported in Table~\ref{tab:params} and Table~\ref{tab:params_synth} for synthetic datasets.

\noindent \textbf{Baselines.}
We report additional details on the implementation of the baseline models and on the hyperparameter grids that we use for selection on a validation set.
\begin{itemize}
    \item For ridge regression, we use the \verb|Ridge| class in the scikit-learn library and we tune its regularization parameter $\alpha_R$ among the following values: $\{1, 10, 100, 200\}$.
    \item For VNN, we optimize among the following parameters: number of layers $L \in \{ 1,2 \}$, embedding size $F \in \{ 16,32 \}$, filters order $K \in \{ 1,2,3 \}$. We use a learning rate of $0.01$, LeakyReLU as non-linear activation, full batch and we train for 500 epochs selecting the best-performing model on the validation set.
    \item For PCA, we select the number of components $k$ among the following values: $k\in \{ 10, 20, 50 \}$.
    \item For MLP, we use the \verb|MLPRegressor| class in the scikit-learn library, we use a feature size $F=32$ and we tune the other hyperparameters in the following ranges: number of layers $L \in \{ 1,2 \}$, L2 regularization term $\alpha_R \in \{1, 10, 100, 200\}$.  
\end{itemize}

\noindent \textbf{CSTs.}
For the CSTs, we search among the following values: $J \in \{ 4, 5, 6, 7 \}$, $L \in \{ 2,3,4 \}$, $\mtT \in \{\mtC_N, \mtC_I\}$.
We use absolute value as nonlinearity $\rho$.
For the ridge regression readout, we optimize the regularization coefficient $\alpha_R$ among the values $\{ 1, 10, 100, 200 \}$.
For CSTs, we set $U=\mtI$ in all experiments except those where we analyze aggregation for dimensionality reduction where we set $U=\textbf{1}_N/N$ (average). 
For monic cubic wavelets, we use the parameters $\alpha=2$, $\beta=2$ and $K=20$ as in~\cite{hammond2011wavelets}. For Hann wavelets, we use $R=3$ and we perform warping as in~\cite{gama2019stabilityscatter}.

\begin{figure}[t]
    \centering
    \begin{subfigure}{.49\textwidth}
    \includegraphics[width=1\linewidth]{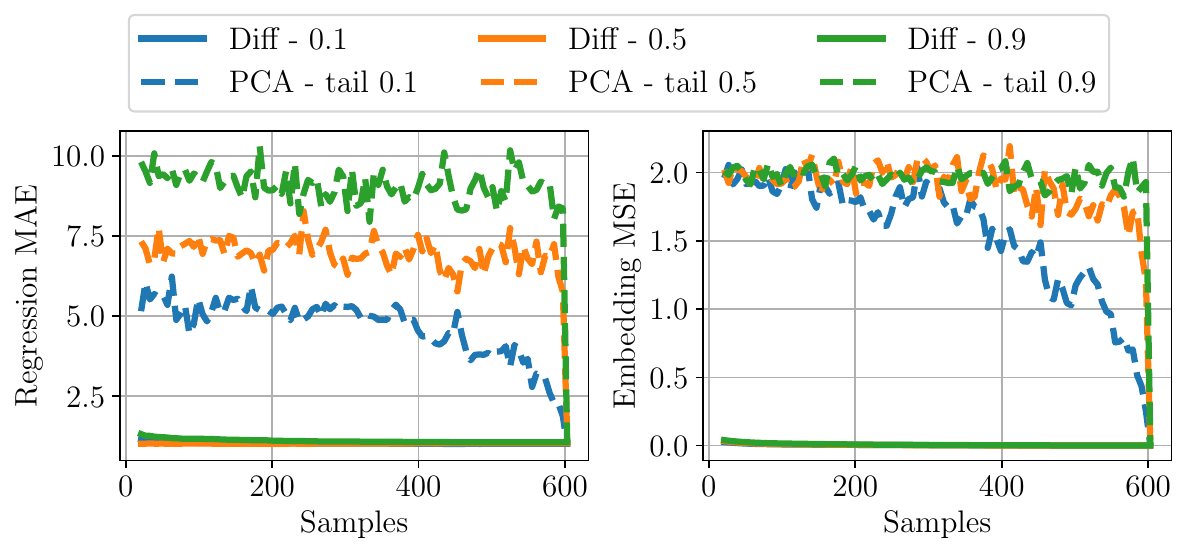} 
    \end{subfigure}
    \hfill
    \begin{subfigure}{.49\textwidth}
    \includegraphics[width=1\linewidth]{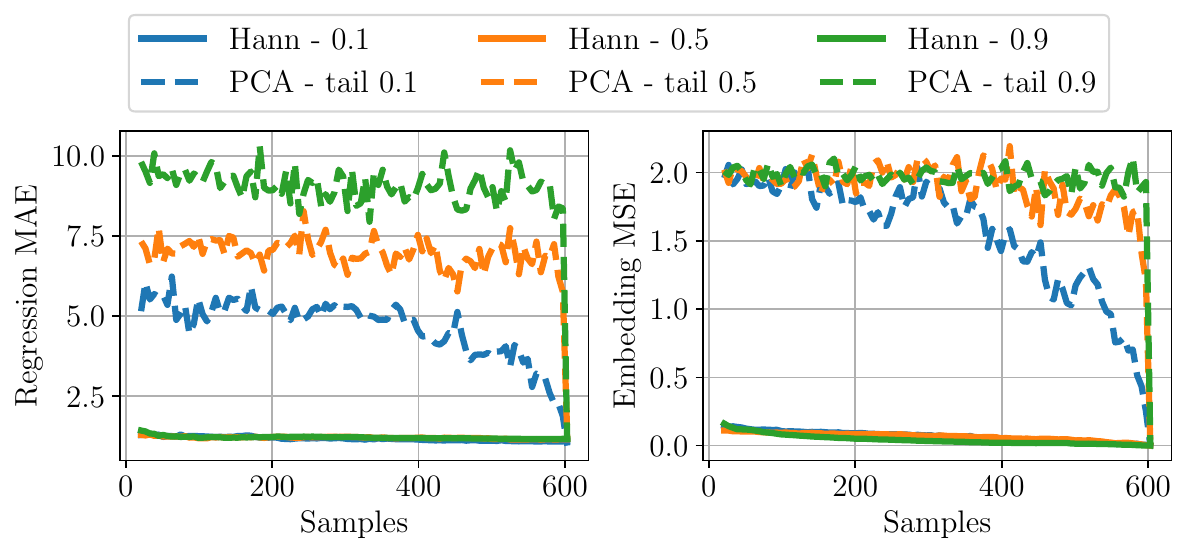} 
    \end{subfigure}
    \hfill
    \begin{subfigure}{.49\textwidth}
    \includegraphics[width=1\linewidth]{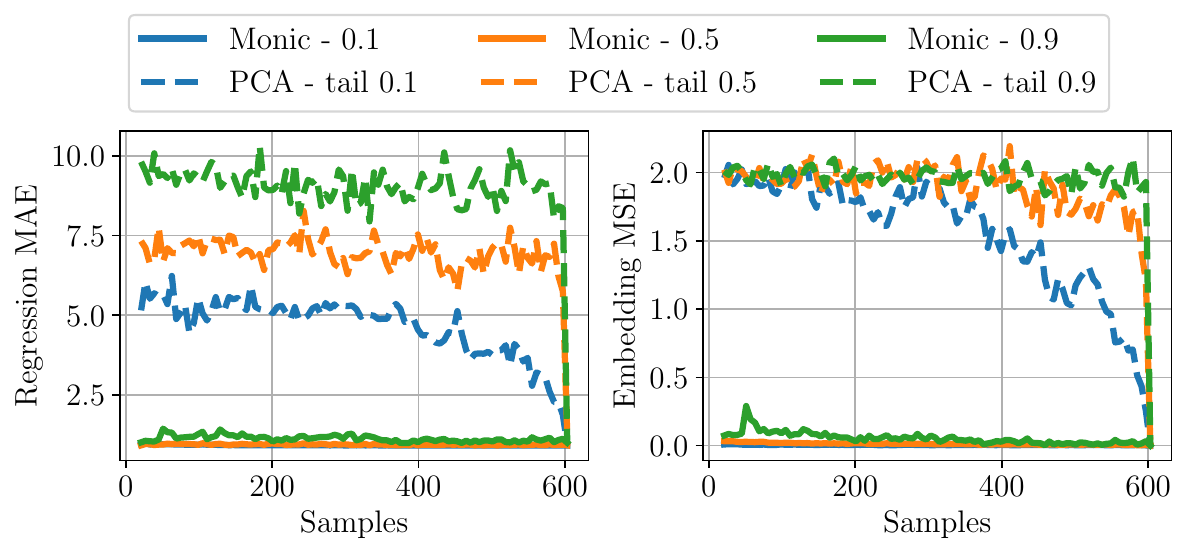} 
    \end{subfigure}
    \caption{Regression performance and embedding perturbation of CSTs and PCA under covariance estimation errors for datasets with different tail strengths.} 
    \label{fig:synth_tail}        
\end{figure}

\begin{table*}[t]
\centering
% \footnotesize
\begin{tabular}{c|l|cccc}
\toprule
& & \textbf{ADNI1} & \textbf{ADNI2} & \textbf{PPMI} & \textbf{ABIDE} \\
\midrule
\multirow{5}{*}{\textbf{Baselines}} 
& \textbf{Ridge} & 5.12~$\pm$~0.14 & \underline{\textit{5.44}}~$\pm$~0.24 & 6.57~$\pm$~0.20 & 4.28~$\pm$~0.22 \\
& \textbf{PCA+Ridge} & 5.14~$\pm$~0.15 & 5.47~$\pm$~0.26 & 6.58~$\pm$~0.19 & 4.30~$\pm$~0.23 \\
& \textbf{KPCA+Ridge} & 5.14~$\pm$~0.13 & 5.48~$\pm$~0.25 & 6.75~$\pm$~0.20 & 4.30~$\pm$~0.24 \\
& \textbf{MLP} & 6.61~$\pm$~0.73 & 6.61~$\pm$~0.45 & 7.29~$\pm$~0.37 & \textbf{3.88~$\pm$~0.10} \\
& \textbf{VNN} & 5.18~$\pm$~0.16 & 5.69~$\pm$~0.23 & 6.96~$\pm$~0.19 & 6.08~$\pm$~0.30 \\
\midrule
\multirow{3}{*}{\textbf{Ours}} 
& \textbf{Diff CST} & \textbf{5.06~$\pm$~0.12} & 5.46~$\pm$~0.24 & \textbf{6.49}~$\pm$~0.21 & 4.34~$\pm$~0.20 \\
& \textbf{Hann CST} & \underline{\textit{5.09}}~$\pm$~0.12 & \textbf{5.43}~$\pm$~0.25 & 6.58~$\pm$~0.21 & \underline{\textit{4.03}}~$\pm$~0.26 \\
& \textbf{Monic CST} & \underline{\textit{5.09}}~$\pm$~0.13 & 5.48~$\pm$~0.24 & \underline{\textit{6.57}}~$\pm$~0.23 & 4.14~$\pm$~0.24 \\
\bottomrule
\end{tabular}
\caption{Regression MAE ($\downarrow$) for CSTs and baselines. \textbf{First best} and \underline{\textit{second best}} results are highlighted.}
\label{tab:extra_comparisons}
\end{table*}

\subsection{Setup for Stability Experiments}

We provide additional details on the setup of the stability experiments on real and synthetic datasets.
We keep $50\%$ of the dataset as unlabeled, i.e., we do not consider the age information for these patients (or other regression target) but we only use their cortical thickness measures to compute the unsupervised representations via PCA or CSTs. The remaining portion of the data is used as labeled dataset for the downstream regression task and is split into sets of size $10\%$ for training, $20\%$ for validation and $20\%$ for test. 
During the training phase, we compute the CSTs on the union of training and unlabeled set $\mathcal{U}$ and we produce representations for the training, which are fed to a ridge regressor to predict the patient's age. Once the ridge regressor is trained, we fix its coefficients. 
Then, we simulate finite-sample covariance estimation errors by re-computing the CST on a covariance estimated from a subset of $\mathcal{U}$. With these perturbed CSTs, we produce perturbed representations of the test samples, and we feed them to the previously trained ridge regressor for the downstream task.
We analyze the stability in terms of MAE on the regression task with perturbed representations and MSE between perturbed and clean representations.

\section{Additional Experiments and Ablations}

We report additional experiments to further investigate the objectives \textbf{(O1)}, \textbf{(O2)}, \textbf{(O3)} in the main body.

\begin{figure*}[t]
    \centering
    \includegraphics[width=1\linewidth]{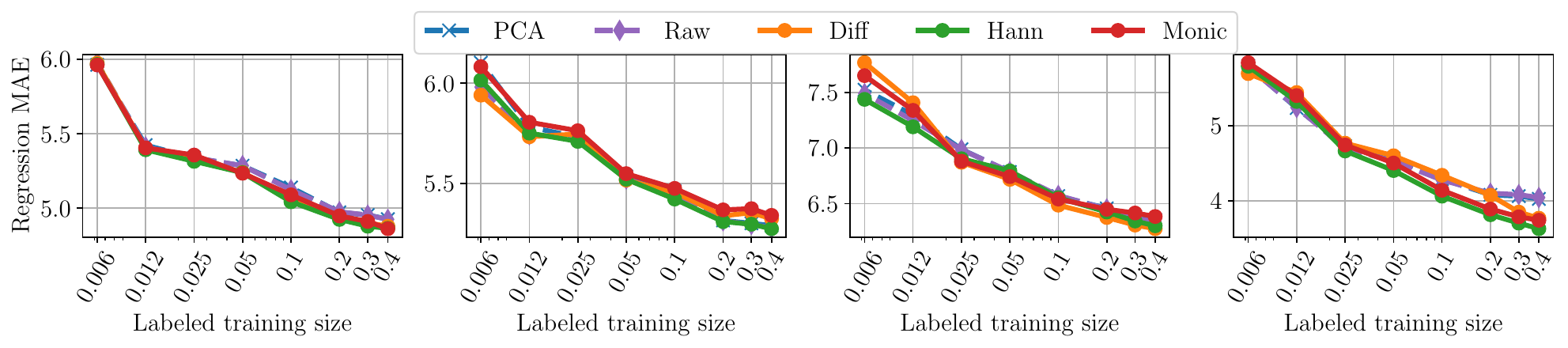}
    \caption{Impact of the size of the labeled training set on performance. Datasets are, from left to right, ADNI1, ADNI2, PPMI and Abide. }
    \label{fig:labeled_exp}        
\end{figure*}

\begin{figure*}[t]
    \centering
    \includegraphics[width=1\linewidth]{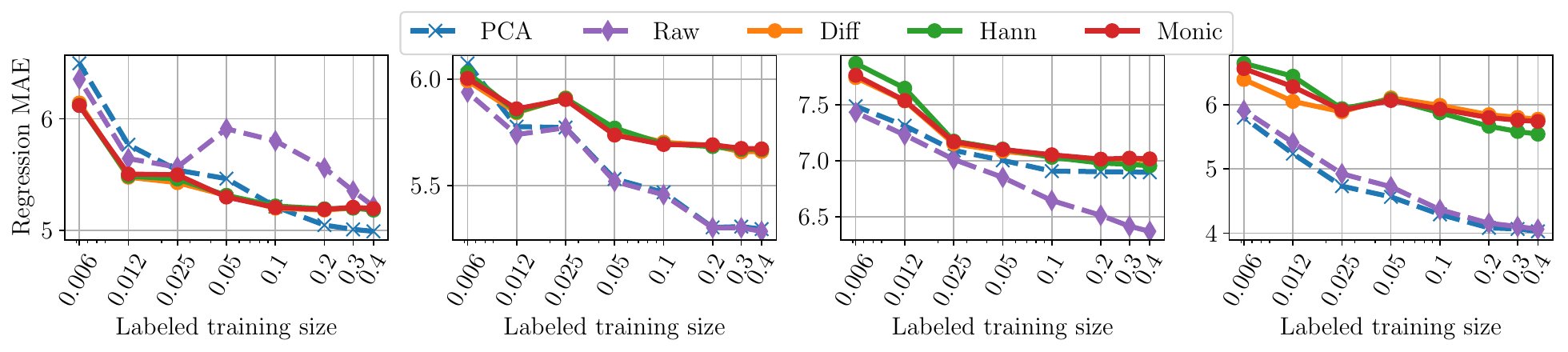}
    \caption{Impact of the size of the labeled training set on performance when using mean aggregation. Datasets are, from left to right, ADNI1, ADNI2, PPMI and Abide. }
    \label{fig:aggregation}        
\end{figure*}

\subsection{Stability on Synthetic Datasets}

We perform experiments on a synthetic regression task to better evaluate the stability of CST for covariances with different eigenvalue distributions \textbf{(O1)}.

\noindent \textbf{Datasets.}
We generate 3 synthetic regression datasets of $T=1000$ observations of size $N=20$ with different covariance eigenvalue distributions via the \verb|make_regression| function in the scikit-learn library~\cite{scikit-learn}, which controls the eigenvalue distribution tail strength. In particular, covariances with larger tail have closer eigenvalues, as illustrated in Figure~\ref{fig:eig_distr}, and represent more difficult settings for PCA (c.f. \eqref{eq:pca_bound}).

\noindent \textbf{Discussion.} 
Figure~\ref{fig:synth_tail} shows that CSTs produce stable representations under covariance estimation errors, which are not strongly affected by the eigenvalue tail. PCA, instead, is highly sensitive to covariance perturbations, leading to unstable results, and is strongly affected by the eigenvalue closeness, which is in line with the theoretical bound in~\eqref{eq:pca_bound}. 
This results in CSTs having a consistent performance on the regression task under covariance estimation errors and producing embeddings close to the ones computed on the clean covariance, whereas PCA's embeddings significantly degrade with noisy covariances and lead to an increase in regression error.
This highlights the effectiveness of CSTs in low-sample regimes and heavy-tailed distributions which are problematic for traditional methods like PCA.

\subsection{Age Prediction from Cortical Thickness}

To further investigate the expressiveness of CSTs representations on the 4 datasets of cortical thickness measures, we report comparisons with additional baselines for the age prediction task and details on the standard deviations in Table~\ref{tab:extra_comparisons}. The representations produced by CSTs fed to a ridge regressor produce the best results on 3 of the 4 datasets, whereas they rank second best on Abide after the MLP. However, the MLP performs poorly on the other 3 datasets, likely due to its higher training complexity for this low-training-data setting, making CSTs's performance the most consistent across datasets.
This corroborates the expressivity of CSTs' representations and their capability to capture relevant patterns from unlabeled data.

\subsection{Impact of Labeled Data and Aggregation}

We analyze the impact of the size of unlabeled dataset for unsupervised learning when using two different aggregation functions $U$: identity $\mtI$ and mean aggregation $\mathbf{1}_N/N$. In the first case, the results can be compared with those in Table~\ref{tab:extra_comparisons} and Figure~\ref{fig:stab_real}, whereas in the second case we assess the impact of aggregation for dimensionality reduction as the number of labels for supervised training becomes smaller. 
For the experiments with mean aggregation, we set the number of components of PCA to match the number of features of CST to ensure a comparison among dimensionality reduction techniques with an equal number of coefficients.

\noindent \textbf{Experimental setup.} For $U=\mtI$, we pick the best hyperparameters computed in Table~\ref{tab:extra_comparisons}, while for $U=\mathbf{1}_N/N$, we perform a different hyperparameter search using a labeled training set of $0.6\%$ to evaluate the impact of dimensionality reduction under limited labels. 
Then, using the best hyperparameters found, we re-compute CSTs, PCA and we re-train ridge regressors on a labeled training set of size varying from $0.6\%$ to $40\%$, while the validation and test sets remain fixed to $20\%$. The remaining portion of the data is used as unlabeled (i.e., only for PCA and CSTs representations). 

\noindent \textbf{Discussion.} Figure~\ref{fig:labeled_exp} shows that, for $U=\mtI$, CSTs' representations lead to similar or better performance than PCA and raw features for most labeled training sizes, indicating their effectiveness under different training scenarios.
For very small training sizes, however (e.g., $0.6\%$) the increased representation size of CSTs causes a performance drop. To counteract this issue, $U$ can be set to perform dimensionality reduction, such as the mean aggregation operation $\mathbf{1}_N/N$. As shown in Figure~\ref{fig:aggregation}, in this case the CSTs match or outperform PCA on ADNI1 and ADNI2 for very small labeled training size due to their reduced representation size. On ADNI1, this aggregation leads to better performance than PCA and raw features even when increasing the labeled training size (where results for raw and PCA are worse than in Table~\ref{tab:extra_comparisons} since the parameters are now optimized in the low-label setting). On the other datasets, instead, increasing the labeled training size leads to PCA and raw features performing better than CSTs due to their increased expressiveness and information richness that the aggregation operation of CSTs loses and that the regressor can learn to exploit.
This identifies a tradeoff where the choice of $U$ can depend on the amount of labeled data available for training on a downstream task.

\subsection{Impact of Pruning}

We elaborate further on the experiment to assess the impact of pruning that we discuss in the main body.

\noindent \textbf{Experimental setup.} Given the best parameter configuration from Table~\ref{tab:params}, obtained without pruning, we assess the impact of our pruning strategy by increasing $\tau$ and checking the regression MAE of CSTs, their computation time (i.e., the time required to compute their output representations) and the dimensionality of their output representations.

\noindent \textbf{Discussion.} Figure~\ref{fig:pruning} shows that, in general, increasing $\tau$ leads to no or minor differences in MAE, while the computation time decreases significantly, often becoming comparable or lower than that of PCA. When $\tau$ gets closer to 1, all scattering features are removed except the original node features, which removes the advantage of CSTs. However, for all datasets there is an intermediate value of $\tau$ such that the MAE remains the same (improving w.r.t. PCA and raw features), but computational time and number of features are reduced, corroborating the effectiveness of the pruning mechanism to reduce complexity while maintaining the expressivity advantages of CSTs.

\begin{figure*}[t]
    \centering
    \begin{subfigure}{\textwidth}
    \includegraphics[width=1\linewidth]{figures/pruning_adni1.pdf} 
    \end{subfigure}
    \hfill
    \begin{subfigure}{\textwidth}
    \includegraphics[width=1\linewidth]{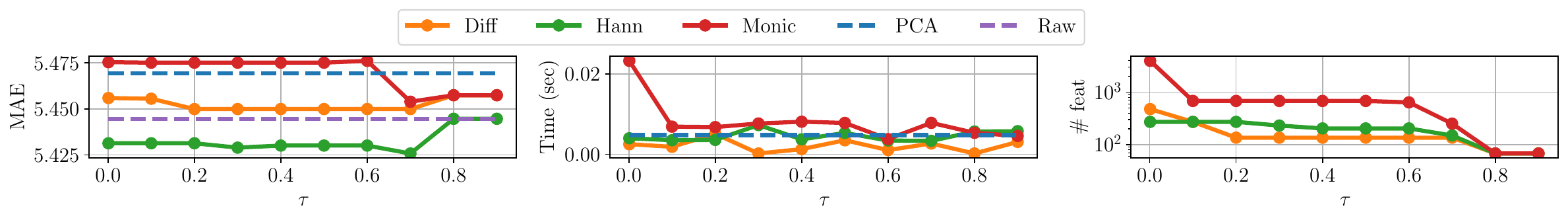} 
    \end{subfigure}
    \hfill
    \begin{subfigure}{\textwidth}
    \includegraphics[width=1\linewidth]{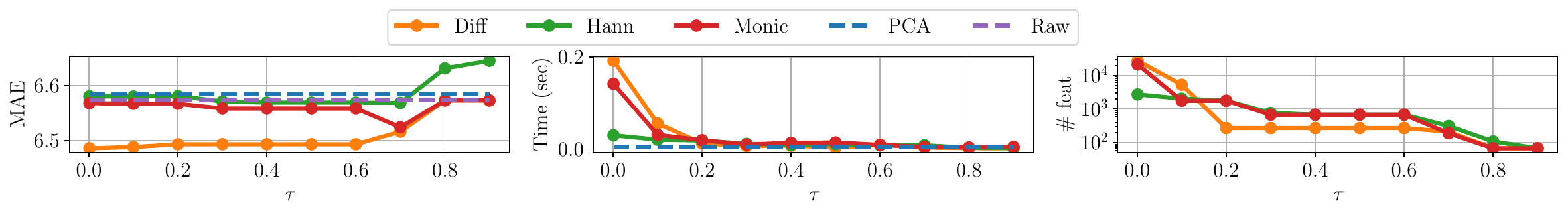} 
    \end{subfigure}
    \hfill
    \begin{subfigure}{\textwidth}
    \includegraphics[width=1\linewidth]{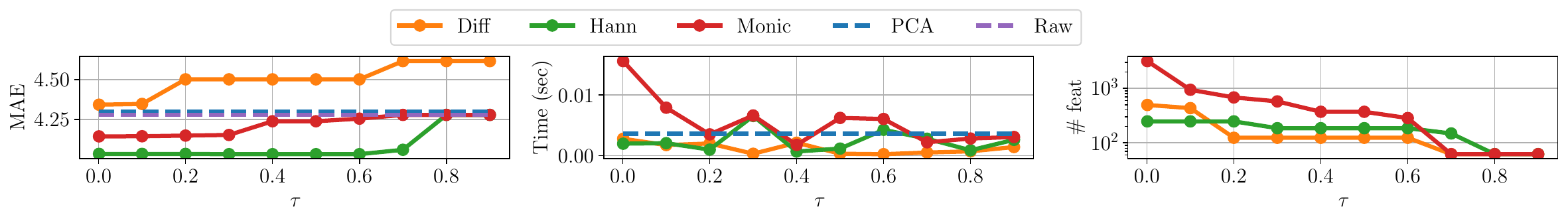} 
    \end{subfigure}
    \caption{Impact of pruning. From top to bottom, the datasets (one per row) are ADNI1, ADNI2, PPMI and Abide. The columns, from left to right, show the regression MAE on the downstream age prediction task, the time to compute the output representations and the number of features retained by the pruning.} 
    \label{fig:pruning}        
\end{figure*}

\end{document}